\documentclass{article}

    \usepackage{arxiv}

\usepackage[utf8]{inputenc} % allow utf-8 input
\usepackage[T1]{fontenc}    % use 8-bit T1 fonts
\usepackage{hyperref}       % hyperlinks
\usepackage{url}            % simple URL typesetting
\usepackage{booktabs}       % professional-quality tables
\usepackage{amsfonts}       % blackboard math symbols
\usepackage{nicefrac}       % compact symbols for 1/2, etc.
\usepackage{microtype}      % microtypography
\usepackage{epsfig}
\usepackage{graphicx}
\usepackage{amsmath}
\usepackage{dsfont}
\usepackage{amssymb}
\usepackage{subcaption}
\usepackage{enumitem}
\usepackage{hyperref}
\usepackage{xcolor}
\usepackage{hyperref}
\usepackage{tikz}
\usepackage{physics}
\usepackage{mathdots}
\usepackage{yhmath}
\usepackage{cancel}
\usepackage{color}
\usepackage{siunitx}
\usepackage{array}
\usepackage{multirow}
\usepackage{gensymb}
\usepackage{wrapfig}
\usepackage{lipsum}
\usepackage{makecell}
\usepackage{nccmath}
\usepackage[export]{adjustbox}
\usepackage{tabularx}
\usetikzlibrary{fadings}
\usetikzlibrary{patterns}
\usetikzlibrary{shadows.blur}
\usetikzlibrary{shapes}

\usepackage{amsthm}
\usepackage{pythonhighlight}
\usepackage{todonotes}
\usepackage{changepage}
\usepackage{natbib}

\usepackage{hyperref}
\usepackage{url}

\usepackage[capitalise]{cleveref}

\usepackage{algorithm}
\usepackage{algorithmic}
\usepackage{soul}
\renewcommand{\algorithmiccomment}[1]{\bgroup\hfill//~#1\egroup}

\newtheorem{lemma}{Lemma}%[section]
\newtheorem{theorem}{Theorem}%[section]
\newtheorem{definition}{Definition}%[section]

\definecolor{L}{RGB}{255,190,190}
\definecolor{Ls}{RGB}{223,255,190}
\definecolor{Lt}{RGB}{190,255,255}
\definecolor{I}{RGB}{223,190,255}

\newcommand{\rv}[1]{\mathbf{#1}}
\newcommand{\z}{\rv{z}}
\newcommand{\x}{\rv{x}}

\newcommand{\setstyle}[1]{\mathbb{\uppercase{#1}}}
\newcommand{\smallsetstyle}[1]{{\scriptscriptstyle \setstyle{#1}}}
\newcommand{\wk}[1]{\rv{w}_\smallsetstyle{#1}}
\newcommand{\zk}[1]{\z_\smallsetstyle{#1}}

\newcommand{\Jk}[1]{J_\smallsetstyle{#1}}
\newcommand{\Gk}[1]{G_\smallsetstyle{#1}}
\newcommand{\Ak}[2]{#1_\smallsetstyle{#2}}

\newcommand{\Lk}[1]{\mathcal{L}_\smallsetstyle{#1}}
\newcommand{\Lhatk}[1]{\widehat{\mathcal{L}}_\smallsetstyle{#1}}
\newcommand{\fcontour}[1]{f_\smallsetstyle{#1}(\zk{#1})}
\newcommand{\pmi}[2]{\mathcal{I}_{\smallsetstyle{#1},\smallsetstyle{#2}}}
\newcommand{\pmihat}[2]{\widehat{\mathcal{I}}_{\smallsetstyle{#1},\smallsetstyle{#2}}}

\newcommand{\APk}[2]{#1_{\mathcal{P}_{#2}}}

\newcommand{\Ip}{\mathcal{I}_\mathcal{P}}
\newcommand{\Ihatp}{\widehat{\mathcal{I}}_{\mathcal{P}}}

\newsavebox{\measurebox}

\title{Principal manifold flows}

\author{Edmond Cunningham \\
  \texttt{edmondcunnin@cs.umass.edu} \\
   \And
   Adam Cobb \\
   \texttt{adam.cobb@sri.com} \\
   \AND
   Susmit Jha \\
   \texttt{susmit.jha@sri.com} \\
}

\begin{document}
\maketitle

\begin{abstract}
% Normalizing flows map an independent set of latent variables to their samples using a bijective transformation.  Despite the exact correspondence between samples and latent variables, their high level relationship is not well understood.  In this paper we characterize the geometric structure of the distribution generated by flows using their principal manifolds and understand the relationship between latent variables and samples using contours, which are the space of samples when all but one latent variable is fixed.  We introduce a novel class of normalizing flows, called principal manifold flows (PF), whose contours are its principal manifolds.  When a PF learns a dataset's true data generating distribution, its latent variables align with the dataset's intrinsic structure.  Additionally, we introduce a variant of PFs for learning densities on manifolds which we call injective PFs (iPF).  We construct PFs and iPFs with standard flow architectures and train them using a regularized maximum likelihood objective.  Notably, the objective for iPFs does not require computing a difficult Jacobian determinant like standard injective flows do.  Our experiments demonstrate that PFs can learn the distribution and structure of low dimensional data and iPFs can learn the distribution and structure of high dimensional data embedded on a low dimensional manifold.
Normalizing flows map an independent set of latent variables to their samples using a bijective transformation.  Despite the exact correspondence between samples and latent variables, their high level relationship is not well understood.  In this paper we characterize the geometric structure of flows using principal manifolds and understand the relationship between latent variables and samples using contours.  We introduce a novel class of normalizing flows, called principal manifold flows (PF), whose contours are its principal manifolds, and a variant for injective flows (iPF) that is more efficient to train than regular injective flows.  PFs can be constructed using any flow architecture, are trained with a regularized maximum likelihood objective and can perform density estimation on all of their principal manifolds.  In our experiments we show that PFs and iPFs are able to learn the principal manifolds over a variety of datasets.  Additionally, we show that PFs can perform density estimation on data that lie on a manifold with variable dimensionality, which is not possible with existing normalizing flows.
\end{abstract}

\section{Introduction}
A normalizing flow is a generative model that generates a probability distribution by transforming a simple base distribution into a target distribution using a bijective function \citep{pmlr-v37-rezende15,papamakarios_normalizing_2019}.  Despite the fact that flows can compute the log likelihood of their samples exactly and associate a point in the data space with a unique point in the latent space, they are still poorly understood as generative models.  This poor understanding stems from the unidentifiablity of their latent space - the latent space of a flow can be transformed into another valid latent space using any one of an infinite number of volume preserving transformations \citep{nonlinear_ica}.  Furthermore, methods that are rooted in mutual information \citep{alemi_fixing_2018,higgins_beta-vae_2017,chen2016infogan} that are used to understand other kinds of generative models break down when applied to flows because there is a deterministic mapping between the latent and data space \citep{ardizzone_training_2020}.  To understand the generative process of flows, we need two items.  The first is a set of informative structural properties of a probability distribution and the second is knowledge of how changes to the latent variables affect corresponding samples.  We discuss the former using the concept of principal manifolds and the latter using contours.

The principal manifolds of a probability distribution can be understood as manifolds that span directions of maximum change \citep{gorban2008principal}.  We locally define them using principal components which are the orthogonal directions of maximum variance around a data point, the same way that the principal components used in PCA \citep{Jolliffe2011} are the orthogonal directions of maximum variance of a Gaussian approximation of a dataset.  Principal manifolds capture the geometric structure of a probability distribution and are also an excellent fit for normalizing flows because it is possible to compute the principal components of a flow due to the bijective mapping between the latent and data spaces (see \cref{definition:principal components of a flow}).

The relationship between changes to latent variables and the effect on the corresponding sample in the data space can be understood through the contours of a flow.  The contours of a flow are manifolds that trace the path that a sample can take when only some latent variables are changed.  An important relationship we investigate is how the probability density on the contours relates to the probability density under the full model.  This insight gives us a novel way to reason about how flows assign density to its samples.

We introduce a class of normalizing flows called principal manifold flows (PFs) whose contours are its principal manifolds. We develop deep insights into the generative behavior of normalizing flows that help explain how flows assign density to data points.  This directly leads to a novel test time algorithm for density estimation on manifolds that requires no assumptions about the underlying data dimensionality.  Furthermore, we develop two new algorithms that tackle separate important problems.  The first is a learning algorithm to train PFs using any flow architecture, even those that can be difficult to invert, at a comparable cost to standard maximum likelihood.  The second is an algorithm to train injective PFs that optimize a regularized maximum likelihood objective without needing to optimize a computationally expensive term found in the injective change of variables formula \citep{gemici_normalizing_2016}.  In our experiments we demonstrate the capabilities of PFs by learning the principal manifolds of low dimensional data and high dimensional data that is embedded on a low dimensional manifold, and show that PFs can learn the density of data that lies on a variable dimensional manifold - a task not possible using existing flow based methods.  To summarize, our contributions are as follows:
\begin{enumerate}
    % \item Provide novel insights into the generative behavior of normalizing flows.
    \item We introduce a novel class of flows called PFs whose contours are principal manifolds and propose an efficient learning algorithm.
    \item To overcome the computational cost of computing the expensive Jacobian determinant for PFs, we introduce iPFs as an approach for extending PFs to higher dimensional problems.
    % \item A key contribution is the iPF, which can be trained to maximize the likelihood of data without computing an expensive Jacobian determinant.
    \item We introduce the first flow based solution to learning densities on manifolds with varying dimensionality.
\end{enumerate}

\begin{figure}[t]
    \centering
    \includegraphics[width=0.4\textwidth]{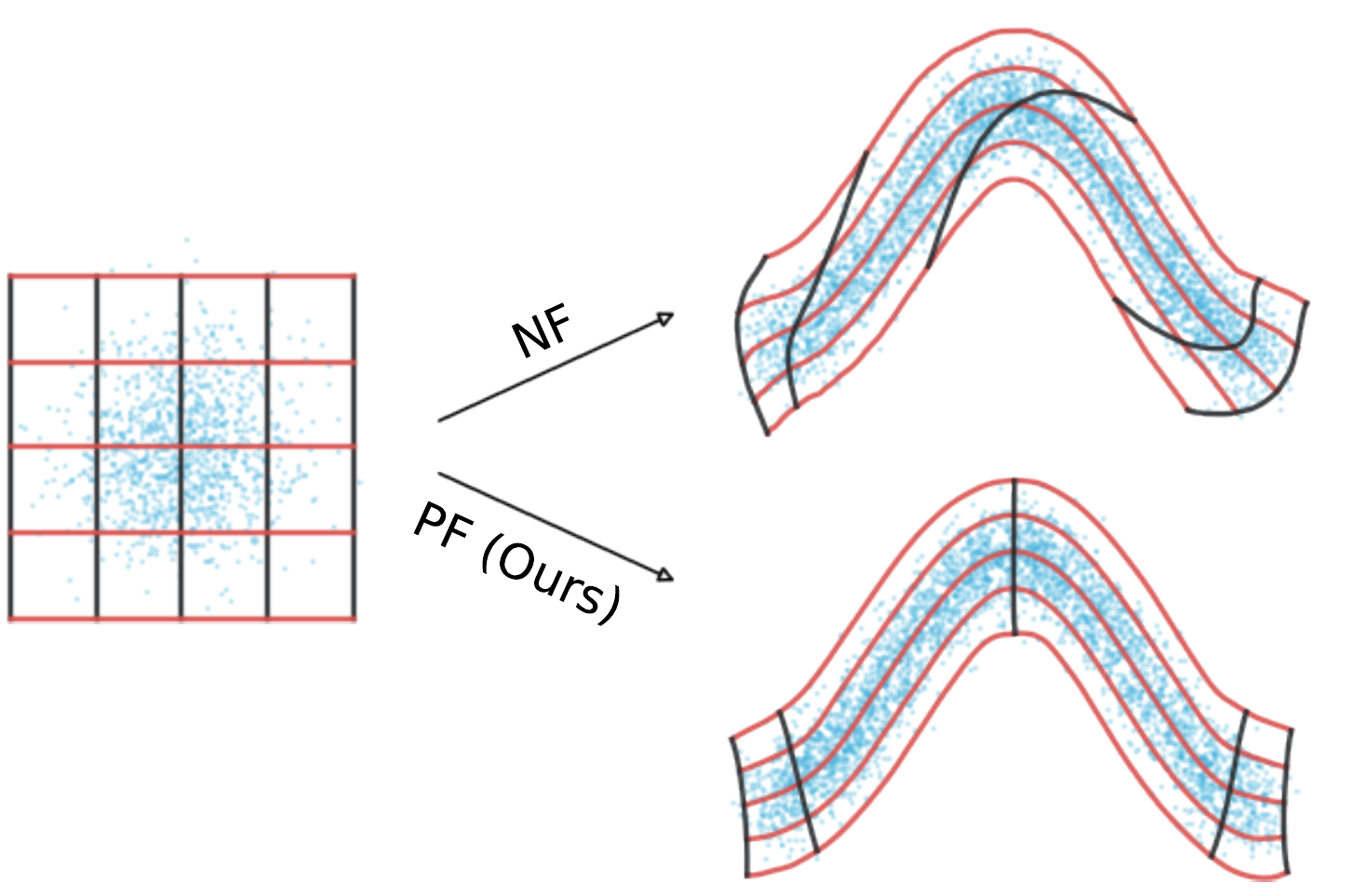}
    \caption{A principal manifold flow (PF) has contours that are its principal manifolds while standard normalizing flows do not.  The blue dots represent samples from the prior and from each model, on the left and right plots respectively.  The red and black lines represent the contours that emerge when a latent variable is held constant and the others vary.  Movement in the latent space of a PF corresponds to movement along the principal manifolds (\cref{theorem:principal manifolds flow}).}
    \label{fig:simple example}
\end{figure}

\section{Preliminaries}

\subsection{Normalizing flows}\label{subsection:nf}
Let $f:\mathcal{Z}=\mathbb{R}^N \to \mathcal{X}=\mathbb{R}^N$ be a parametric bijective function from a latent variable, $\z$, to a data point $\x=f(\z)$ with inverse $g(\x)=f^{-1}(\x)$.  The prior distribution over $\z$ will be denoted with $p_\z(\z)$ and the Jacobian matrix of $f$ and $g$ will be denoted by $J=\frac{df(\z)}{d\z}$ and $G=\frac{dg(\x)}{d\x}$. The dependence of $J$ and $G$ on $\z$ or $\x$ is implied.  A normalizing flow is a model that generates data by sampling $\z\sim p_\z(\z)$ and then computing $\x=f(\z)$ \citep{pmlr-v37-rezende15,papamakarios_normalizing_2019}.  The probability density of data points is computed using the change of variables formula:
\begin{align}\label{eq:change of variables}
    \log p_\x(\x) = \log p_\z(g(\x)) + \log|G|
\end{align}
Flows are typically comprised of a sequence of invertible functions $f=f_1 \cdots f_i \cdots f_K$ where each $f_i$ has a Jacobian determinant that is easy to compute so that the overall Jacobian determinant, $\log|G|=\sum_{i=1}^K \log|G_i|$ is also easy to compute.  As a result, flows can be trained for maximum likelihood using an unbiased objective.

\cref{eq:change of variables} can be generalized to the case where $f$ is an injective function that maps from a low dimensional $\z$ to a higher dimensional $\x$ \citep{gemici_normalizing_2016,caterini_rectangular_2021}.  The change of variables formula in this case is written as:
\begin{align}\label{eq:change of variables general}
    \log p_\x(\x) = \log p_\z(\z) - \frac{1}{2}\log|J^TJ|,\quad \x=f(\z)
\end{align}
This general change of variables formula is valid over the manifold defined by $f(\mathcal{Z})$.  However, it is difficult to work with because the term $\log|J^TJ|$ cannot be easily decomposed into a sequence of simple Jacobian determinants as in the case where $\dim(\z)=\dim(\x)$.  In the remainder of this paper, $\log p_\x(\x)$ will refer to the definition given in \cref{eq:change of variables general} unless stated otherwise.

The requirement that $f$ is bijective is a curse and a blessing.  The constraint prohibits flows from learning probability distributions with topology that does not match that of the prior \citep{cornish2019relaxing}. This constraint limits a flow's ability to learn the exact distribution of many real world datasets, including those that are thought to satisfy the manifold hypothesis \citep{fefferman_testing_2013}.  Nevertheless, invertibility makes it possible to compute the exact log likelihood under the model, associate any data point with a unique latent space vector, and affords access to geometric properties of the flow's distribution \citep{dombrowski2021diffeomorphic}.  In the remainder of this paper we focus on the latter - the geometric properties of a flow's distribution through the use of principal manifolds and contours.

\subsection{Principal components of a flow}\label{sec:principal components of flow}

% \begin{figure}[t]
%     \centering
%     % \includegraphics[width=0.47\textwidth]{figures/2d_manifold/linearization.pdf}
%     \includegraphics[width=0.999\linewidth]{figures/2d_manifold/linearization2.pdf}
%     \caption{Illustration of principal components and a principal manifold of a flow.  The blue dots are samples from a normalizing flow, red dots are samples from the Gaussian approximation defined in \cref{definition:principal components of a flow}, black arrows are the principal components and the green line is a principal manifold in \cref{definition:principal manifold of a flow}.  Our main result in \cref{theorem:principal manifolds flow} states that the contours of principal manifold flows are its principal manifolds.}
%     \label{fig:structure}
% \end{figure}

\begin{figure*}[t]
    \centering
    \includegraphics[width=0.65\linewidth]{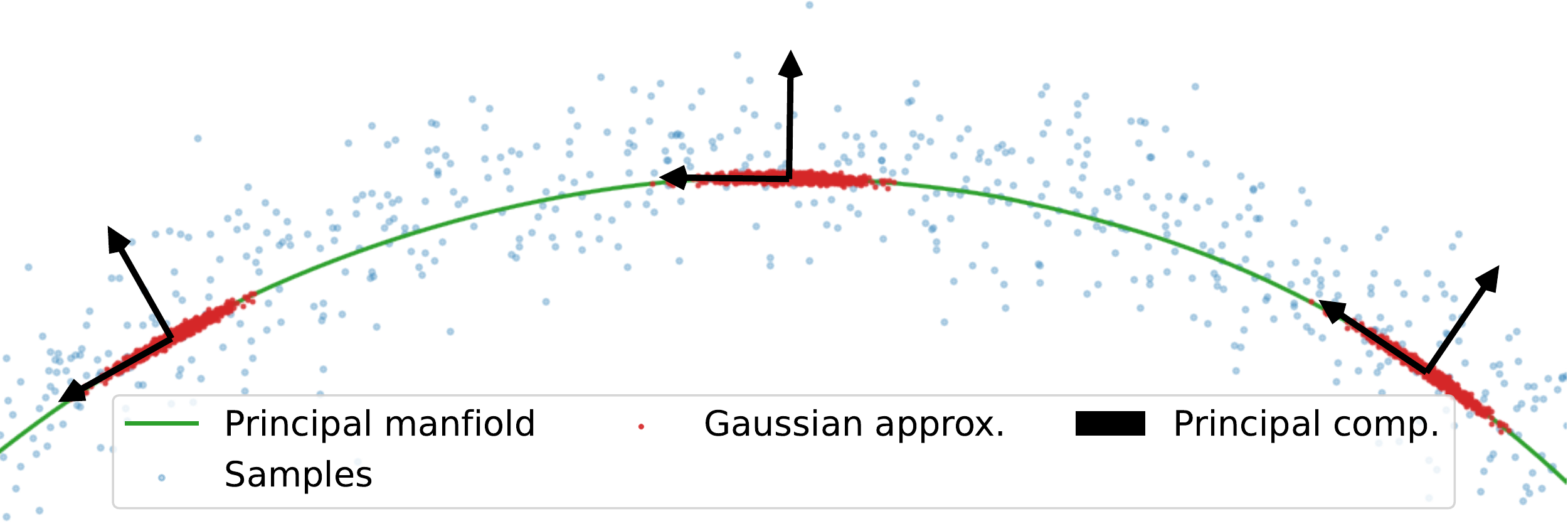}
    \caption{Illustration of principal components and a principal manifold of a flow.  The blue dots are samples from a normalizing flow, red dots are samples from the Gaussian approximation defined in \cref{definition:principal components of a flow}, black arrows are the principal components and the green line is a principal manifold in \cref{definition:principal manifold of a flow}.  Our main result in \cref{theorem:principal manifolds flow} states that the contours of principal manifold flows are its principal manifolds.}
    \label{fig:structure}
\end{figure*}

The structure of a probability distribution that is generated by a normalizing flow can be locally defined by examining how samples from the model are distributed around a data point.

\begin{lemma}\label{lemma:linearization of a flow}
Let $\x$ be a data point, $f$ be the invertible function of a flow and $\sigma>0$ be a scalar.  Consider samples that are generated by $\x'=f(\z+\sigma\rv{\epsilon})$ where $\z=f^{-1}(\x)$ and $\rv{\epsilon} \sim N(0,I)$.  Then $\frac{1}{\sigma}(\x'-\x) \overset{D}{\to} N(0,JJ^T)$ as $\sigma\to 0$.
\end{lemma}
The lemma is true as a direct consequence of the Delta method \citep{doi:10.1080/00031305.1992.10475842}.  \cref{lemma:linearization of a flow} says that points generated by a flow in a small region around a fixed point $\x$ will be approximately distributed as a Gaussian with mean $\x$ and covariance proportional to $JJ^T$.  The principal components of data generated by a Gaussian distribution are the eigenvectors of the covariance matrix, so we can use the eigenvectors of $JJ^T$ to define the principal components of a flow.

\begin{definition}[Principal components of a flow at $x$]\label{definition:principal components of a flow}
The principal components of a flow at $\x=f(\z)$ are the eigenvectors of $JJ^T$, $\hat{\rv{w}}$, where $J=\frac{df(\z)}{d\z}$.  The principal components are ordered according to the eigenvalues of $JJ^T$.
\end{definition}

The concept of principal components is shown in \cref{fig:structure}.  Blue dots represent samples from a flow, red dots are samples from the local approximations drawn according to \cref{lemma:linearization of a flow} and black arrows represent the principal components computed using \cref{definition:principal components of a flow}.  We see that the red dots are approximately distributed as a Gaussian and the black arrows span their principal directions.  Furthermore, the principal components are oriented along the main structure of the data.
The global structure of a flow, which we call the "principal manifolds", are found by integrating along the principal components.
\begin{definition}[Principal manifold of a flow]\label{definition:principal manifold of a flow}
The principal manifold of a flow is the path formed by integrating along principal components starting at $\x_0$.  Let $\setstyle{k}$ be a subset of $[1,\dots,\dim(\z)]$ and $t\in \mathbb{R}^{|\setstyle{k}|}$.  A principal manifold of dimension $|\setstyle{k}|$ is the solution to
\begin{align}
    \frac{d\x(t)}{dt} = \wk{k}(\x(t)),\quad \x(0) = \x_0
\end{align}
where $\wk{k}(\x(t))=\hat{\rv{w}}_{\smallsetstyle{k}}\sqrt{\Lambda_{\smallsetstyle{k}}}$ are the principal components with indices in $\setstyle{k}$ at $\x(t)$ scaled by the square root of their corresponding eigenvalues.
\end{definition}
The green curve in \cref{fig:structure} is one of an infinite number of principal manifolds of the distribution.  It spans the main structure of the samples and has principal component tangents.  Principal manifolds can be used to reason about the geometric structure of a flow, but can only be found via integration over the principal components.  Furthermore, there is no clear way compute the probability density over the principal manifolds.  This is crucial when a principal manifold is used as a low dimensional representation of data and we still want to perform density estimation.  We will revisit principal manifolds in \cref{section:pca flows}.

\subsection{Contours of a normalizing flow}\label{subsection:contour}
The tool we use to analyze the generative properties of flows are the contours that emerge when some latent variables are held constant while others vary.  
\begin{definition}[Contours of a flow]\label{def:contour}
Let $\setstyle{k}$ be a subset of $[1,\dots,\dim(\z)]$ and $\zk{k}$ be the latent variables with indices in $\setstyle{k}$.  Then, the contour obtained by varying $\zk{k}$ and fixing all other variables is denoted as  $\fcontour{k}$.
\end{definition}
We assume that there is a partition over the indices of the latent space, $\mathcal{P}$, so that every set of indices that we use to form contours is an element of the partition: $\setstyle{k} \in \mathcal{P}$.  Additionally, we assume that the prior over $\z$ can be factored into independent components in order to isolate a prior for each contour: $p_\z(\z)=\prod_{\setstyle{k}\in \mathcal{P}}p_\smallsetstyle{k}(\z_\smallsetstyle{k})$.  This is not a limiting assumption as most flow architectures typically use a fully factorized prior such as a unit Gaussian prior \citep{papamakarios_normalizing_2019}.
The curved red and black lines in the right side plots of \cref{fig:simple example} are examples of contours.  A red line on the left side plot is created by varying $\z_1$ and fixing $\z_2$ and becomes the contour $f_1(\z_1)$ after it is passed through the flow.  Similarly, a black line on the left side plot is formed by varying $\z_2$ and fixing $\z_1$ and becomes the contour $f_2(\z_2)$ when it is transformed by the flow.
The Jacobian matrix of $\fcontour{k}$ is denoted by $\Jk{k}$ and is equal the matrix containing the columns of $J$ with indices in $\setstyle{k}$.  The log likelihood of a contour is denoted by $\mathcal{L}_\setstyle{k}$ and is computed using the change of variables formula on manifolds in \cref{eq:change of variables general}:
\begin{align}\label{eq:Lk}
    \Lk{k} \overset{{\scriptstyle \Delta}}{=} \log p(\fcontour{k}) = \log p_{\setstyle{k}}(\zk{k}) - \frac{1}{2} \log|\Jk{k}^T \Jk{k}|
\end{align}
A single flow can assign many different log likelihoods to a given data point that are not given by \cref{eq:change of variables}. Each of the $|\mathcal{P}|$ contours that intersect at a data point can assign a density given by \cref{eq:Lk}. Furthermore, the contours formed by grouping multiple $\zk{k}$ will assign other densities.  In the next section we will relate all of these different densities.

\subsection{Pointwise mutual information between contours}\label{subsection:pmi}
Consider two disjoint subsets of a latent variable, $\zk{s}$ and $\zk{t}$, and their union $\zk{s+t}$.  The densities of each contour can be related to the densities of their union using pointwise mutual information.

\begin{definition}[Pointwise mutual information between disjoint contours]\label{def:pmi}
Let $\zk{s}$ and $\zk{t}$ be disjoint subsets of $z$.  The pointwise mutual information between the contours for $\zk{s}$ and $\zk{t}$ is defined as
\begin{align}\label{eq:pmi}
    \pmi{s}{t} &\overset{{\scriptstyle \Delta}}{=} \log \frac{p(\fcontour{s+t})}{p(\fcontour{s})p(\fcontour{t})} 
    = \Lk{s+t} - \Lk{s} - \Lk{t}
\end{align}
\end{definition}

$\pmi{s}{t}$ plays an important role in describing the behavior of normalizing flows.  We list a few important facts about $\pmi{s}{t}$ below to help build intuition (more facts and proofs can be found in \cref{appendix:contour cookbook}):
\paragraph{Facts about $\pmi{s}{t}$}
\begin{enumerate}
    \item $\pmi{s}{t} \geq 0$ \label{fact: 1}
    \item $\pmi{s}{t}=0$ iff $\fcontour{s}$ and $\fcontour{t}$ intersect orthogonally. \label{fact: 2}
    % \item $\pmi{s}{t}=0$ iff $\Jk{s+t}^T\Jk{s+t}$ is block diagonal. \label{fact: 3}
    \item $\pmi{s}{t}= - \frac{1}{2}\log|\Jk{s+t}^T\Jk{s+t}| + \frac{1}{2}\log|\Jk{s}^T\Jk{s}| + \frac{1}{2}\log|\Jk{t}^T\Jk{t}|$ \label{fact: 4}
\end{enumerate}
$\pmi{s}{t}$ is a non-negative value that achieves its minimum of $0$ when the contours intersect orthogonally.  An equivalent condition is that the columns of $\Jk{s}$ and $\Jk{t}$ are mutually orthogonal. The third fact shows that $\pmi{s}{t}$ has a closed form value that only depends on $f$ and not on the priors over $\zk{s}$ or $\zk{t}$. Another way to think about $\pmi{s}{t}$ is as the difference between the contour log likelihoods and that of their union:
\begin{align}
    \Lk{s+t} = \Lk{s} + \Lk{t} + \pmi{s}{t} \label{eq:general 2 partition}
\end{align}
% It is clear from \cref{eq:general 2 partition} that if $\setstyle{s}+\setstyle{t}=[1,\dots,\dim(\z)]$, then we can decompose the change of variables formula into the sum of contour log likelihoods and $\pmi{s}{t}$.
If $\setstyle{s}+\setstyle{t}=[1,\dots,\dim(\z)]$, then we can decompose the change of variables formula into the sum of contour log likelihoods and $\pmi{s}{t}$.
Next, we show that this decomposition can be extended to any partition of the latent space.

\subsection{Change of variables formula decomposition}\label{subsection:change of variables decomp}

\begin{figure}[t]
    \centering
    \begin{subfigure}[t]{0.45\linewidth}
        \centering
        \includegraphics[width=0.5\textwidth]{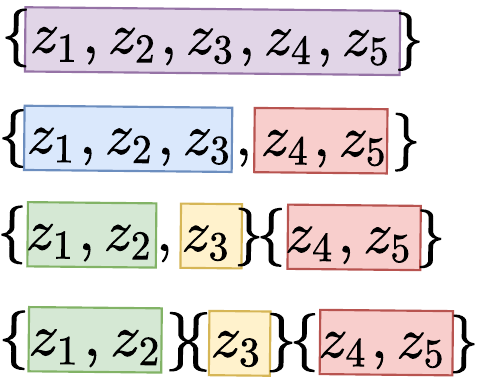}
        \caption{Example partitions of the latent space.}
        \label{fig:partition}
    \end{subfigure}\hspace{0.5cm}%
    \begin{subfigure}[t]{0.45\linewidth}
        \centering
        \includegraphics[width=0.5\textwidth]{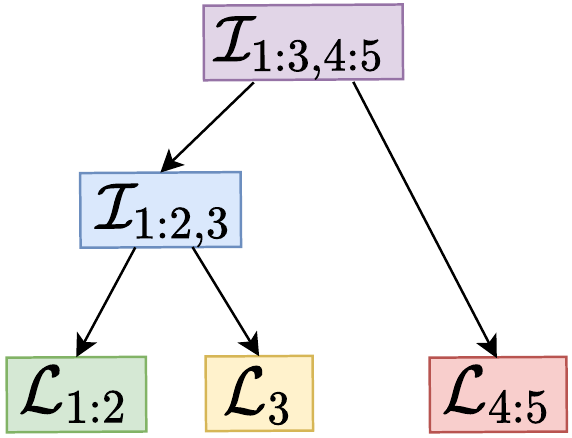}
        \caption{Decomposition of log likelihood.}
        \label{fig:partition terms}
    \end{subfigure}
    \caption{The general change of variables decomposition depends on the binary tree partition used to generate the latent space partition.  Partitions are generated by recursively dividing existing partitions into two parts.  The last row in \cref{fig:partition} shows a partition of the latent space with 3 sets.  \cref{fig:partition terms} shows the corresponding log likelihood decomposition.  Each parent node in the binary tree contributes an $\mathcal{I}$ term and each leaf contributes a $\mathcal{L}$ term.  See \cref{eq:generalized decomposition} for the full formula.}
    \label{fig:partition example}
\end{figure}

\cref{eq:general 2 partition} can be recursively applied to itself to yield a decomposition over any partition of the latent space.  Consider a partition of the indices, $\mathcal{P}$, like in \cref{fig:partition}.  $\mathcal{P}$ can be constructed as the leaves of a binary tree, $\mathcal{T}$, where each node is a subset of indices and each parent node is the union of its children.  The corresponding decomposition in \cref{fig:partition terms} is found by recursively applying \cref{eq:general 2 partition} and tracking the leftover $\mathcal{I}$ terms.  This construction lets us decompose the change of variables formula into the sum of contours log likelihoods and pointwise mutual information terms:

\begin{align} \label{eq:generalized decomposition}
    \log p_\x(f(\z)) = \sum_{\setstyle{k}\in \mathcal{P}}\Lk{k} + \underbrace{\sum_{\setstyle{p} \in \text{parents}(\mathcal{T})}\pmi{\text{l}(p)}{\text{r}(p)}}_{\Ip}
\end{align}

where $\text{L}({p})$ and $\text{R}({p})$ are the left and right children of $\setstyle{p}$, respectively.  See \cref{fig:partition example} for a visual description.  Notice that the sum of the various $\pmi{\text{l}(p)}{\text{r}(p)}$ is independent of the choice of $\mathcal{T}$ because any $\mathcal{T}$ with the same leaves will have the same value of $\log p_\x(f(\z))$ and $\sum_{\setstyle{k}\in \mathcal{P}}\Lk{k}$.  This non-negative quantity is useful to know as it is the difference between the full log likelihood and sum of log likelihoods of the contours.
\begin{definition}[Pointwise mutual information of a partition]
Let $\mathcal{P}$ be a partition of $[1,\dots,\dim(\z)]$.  The pointwise mutual information of the flow whose latent space is partitioned by $\mathcal{P}$ is
% \begin{align}
%     \Ip &\overset{{\scriptstyle \Delta}}{=} \log p_\x(f(\z)) - \sum_{\setstyle{k}\in \mathcal{P}}\Lk{k} \nonumber \\
%     &= -\frac{1}{2}\log|J^TJ| + \frac{1}{2}\sum_{\setstyle{k}\in \mathcal{P}}\log|\Jk{k}^T\Jk{k}|
% \end{align}
\begin{align}
    \Ip &\overset{{\scriptstyle \Delta}}{=} \log p_\x(f(\z)) - \sum_{\setstyle{k}\in \mathcal{P}}\Lk{k}
\end{align}
\end{definition}

\cref{eq:generalized decomposition} gives insight into how normalizing flows assign density to data.  The log likelihood of a data point under a normalizing flow is the sum of the log likelihoods under its contours, and a non-negative term that roughly measures the orthogonality of the contours.  If during training \cref{eq:generalized decomposition} is maximized, as is the case in maximum likelihood learning, then $\Ip$ will surely not achieve its minimum value of $0$.  Therefore changes to different latent variables of a normalizing flow trained with maximum likelihood will likely produce similar changes in the data space.

\subsection{Orthogonality condition using g}
We have seen that $\Ip$ is a non-negative term that achieves it minimum of $0$ when the contours of the flow are orthogonal.  However obtaining its value requires computing columns of the Jacobian matrix of $f(\z)$.  Many expressive normalizing flows layers \citep{huang_convex_2021,chen_residual_2019,vdberg2018sylvester} are constructed so that only $g(\x)$ is easy to evaluate while $f(\z)$ requires an expensive algorithm that can be difficult to differentiate.  We introduce a novel alternate formulation of $\Lk{k}$, $\pmi{s}{t}$ and $\Ip$ that can be computed with $g(\x)$ to mitigate this issue:
\begin{align}
    \Lhatk{k} &\overset{{\scriptstyle \Delta}}{=} \log p_{\setstyle{k}}(\zk{k}) + \frac{1}{2} \log|\Gk{k}\Gk{k}^T| \\
    % \pmihat{s}{t} &\overset{{\scriptstyle \Delta}}{=}  - \frac{1}{2}\log|\Gk{s+t}\Gk{s+t}^T| + \frac{1}{2}\log|\Gk{s}\Gk{s}^T| + \frac{1}{2}\log|\Gk{t}\Gk{t}^T| \\
    \pmihat{s}{t} &\overset{{\scriptstyle \Delta}}{=} \Lhatk{s+t} - \Lhatk{s} - \Lhatk{t} \\
    % \Ihatp &\overset{{\scriptstyle \Delta}}{=} -\frac{1}{2}\log|GG^T| + \frac{1}{2}\sum_{\setstyle{k}\in \mathcal{P}}\log|\Gk{k}\Gk{k}^T|
    \Ihatp &\overset{{\scriptstyle \Delta}}{=} \log p_\x(\x) - \sum_{\setstyle{k}\in \mathcal{P}}\Lhatk{k}
\end{align}
In contrast to $\pmi{s}{t}$ and $\Ip$, $\pmihat{s}{t}$ and $\Ihatp$ are both negative $\pmihat{s}{t},\Ihatp \leq 0$. See \cref{appendix:contour cookbook} for more properties.  The most important of these properties is the following Lemma:

\begin{lemma}\label{lemma:inverse I}
$\Ihatp = 0$ if and only if $\Ip = 0$.
\end{lemma}
We will see that flows that satisfy $\Ip=0$ are of interest, so this lemma provides an equivalent condition that can be computed by any normalizing flow architecture.

\section{Principal manifold flows}\label{section:pca flows}
% Here we present our main contributions.  
We now present our main contributions.
We will first define PFs and discuss their theoretical properties then discuss learning algorithms to train PFs and injective PFs (iPF).

\subsection{PF theory}
Next, we formally define PFs and provide a theorem stating their primary feature (see \cref{appendix:principal manifolds flow theorem} for the proof).

\begin{definition}[principal manifold flow]\label{def:PF}
A principal manifold flow (PF) is a normalizing flow that satisfies $\Ip=0$ at all of its samples.
\end{definition}

% The primary feature of PFs is the following theorem (see \cref{appendix:principal manifolds flow theorem} for a proof):
\begin{theorem}[Contours of PFs]\label{theorem:principal manifolds flow}
The contours of a principal manifold flow are principal manifolds.
\end{theorem}
A compelling byproduct of \cref{theorem:principal manifolds flow} is that PFs can easily evaluate the probability density of their principal manifolds.  As a result, PFs can perform density estimation on manifolds at test time without making any assumptions about the dimensionality of the data manifold.  This is in stark contrast to existing flow based algorithms for density estimation on manifolds where the manifold dimensionality is fixed when the flow is created \citep{brehmer_flows_2020}.  In order to exploit this ability, we need a method to identify which contours correspond to which principal manifolds.   Recall that the principal components are ordered according to the eigenvalues of $JJ^T$.  Consider a PF where the partition size is 1.  Then the diagonal elements of $J^TJ$ will be equal to the top $\dim(\z)$ eigenvalues of $JJ^T$ because $J$ will be the product of a semi-orthogonal matrix and a diagonal matrix (see \cref{claim:4.5}), so $J^TJ$ will be diagonal and its diagonal elements of $J^TJ$ can be used to identify which contour corresponds to which principal manifold.  In the general case, $J^TJ$ is a block diagonal matrix so we can look at the square root of the determinant of each block matrix, $|\Jk{k}^T\Jk{k}|^\frac{1}{2}$.  $|\Jk{k}^T\Jk{k}|^{\frac{1}{2}}$ is how much the density around $\x$ is "stretched" along the contour $\fcontour{k}$ to form the structure present in the data, so contours with small values of $|\Jk{k}^T\Jk{k}|^{\frac{1}{2}}$ correspond to a direction that contributes little to the overall structure.

This check can be used filter out components of log likelihood that are due to small variations such as noise incurred in the data collection process.  For example, consider a PF trained on 2D data and at $\x$ we observe that $\log|J_1^TJ_1|\gg \log|J_2^TJ_2|$. Then the structure of the probability distribution at $\x$ should be primarily aligned with the contour $f_1(\z_1)$, so it might make sense to report the log likelihood of $\x$ on only this contour.  We define this procedure below:

\begin{definition}[Manifold corrected probability density]\label{def:manifold corrected density}
Let $\x$ be a sample from a PF.  The manifold corrected probability density of $\x$ is computed as
\begin{align}
    \log p_\mathcal{M}(\x) = \sum_{\setstyle{k} \in \mathcal{S}} \Lk{k}, \quad \mathcal{S} = \{\setstyle{k} : |\Jk{k}^T\Jk{k}|^{\frac{1}{2}} > \epsilon\}
\end{align}
\end{definition}
In experiment \cref{sec:manifold 3d} we demonstrate an example where PFs correctly learns the density of data generated on a variable dimension manifold.

\subsection{Learning algorithms}\label{subsection:learning algorithm}

\paragraph{PF objective}\mbox{} \\
The PF optimization problem is to minimize the negative log likelihood of data subject to the constraint that $\Ip = 0$:
\begin{align} \label{eq:PF problem} 
    \text{argmin}_\theta -\sum_{\x \in \mathcal{D}} \log p_\x(\x;\theta),\quad \text{s.t.}\quad \Ip(\x;\theta) = 0
\end{align}
We solve this problem with a regularized maximum likelihood objective:
\begin{align}\label{eq:PF lagrangian}
     &\mathrm{L}(\theta) = \sum_{\x \in \mathcal{D}} -\log p_\x(\x;\theta) + \alpha \Ip(\x;\theta) \nonumber \\
     &= \sum_{\z = g(\x),\x \in \mathcal{D}} -\log p_\z(\z) - \frac{\alpha - 1}{2}\log|J(\z;\theta)^TJ(\z;\theta)|\nonumber \\ 
     &\quad \quad \quad \quad + \frac{\alpha}{2}\sum_{\setstyle{k}\in \mathcal{P}}\log|\Jk{k}(\z;\theta)^T\Jk{k}(\z;\theta)| 
\end{align}
where $\alpha$ is a hyperparameter.  Note that in the special case where the partition $\setstyle{k}$ is 1 dimensional and $\dim(\x) = \dim(\z)$, this is the objective function used in \citep{gresele2021independent}.  Although $\mathrm{L}(\theta)$  is a valid objective, it requires that we compute $g(\x)$ and Jacobian-vector products with $f(\z)$, making it impractical for flows where only $g(\x)$ is easy to evaluate.  We remedy this issue by replacing the constraint $\Ip=0$ with $\Ihatp=0$ as per \cref{lemma:inverse I}.  The result is a novel loss function for training flows to have orthogonal contours:
\begin{align}\label{eq:PF objective}
    &\mathrm{L}_\text{PF}(\theta) = \sum_{\x \in \mathcal{D}}-\log p_\x(\x;\theta) - \alpha \Ihatp(\x;\theta) \nonumber \\
    &= \sum_{\x \in \mathcal{D}}-\log p_\z(g(\x;\theta)) - \frac{\alpha + 1}{2}\log|G(\x;\theta)G(\x;\theta)^T| \nonumber \\
    &\quad \quad \quad \quad + \frac{\alpha}{2}\sum_{\setstyle{k}\in \mathcal{P}}\log|\Gk{k}(\x;\theta)\Gk{k}(\x;\theta)^T|
\end{align}

$\mathrm{L}_\text{PF}(\theta)$ is the objective of choice when $\dim(x) = \dim(z)$.  It provides a lightweight change to maximum likelihood training that can be applied to any flow architecture.

\paragraph{iPF objective}\mbox{} \\
Next consider the case where $\dim(\x) > \dim(\z)$.  This appears in problems where we want to learn a low dimensional representation of data \citep{gemici_normalizing_2016,brehmer_flows_2020,caterini_rectangular_2021,pmlr-v108-kumar20a}.  Although we can optimize $\mathrm{L}(\theta)$ to learn a PF, naively optimizing \cref{eq:PF lagrangian} will require optimizing $\log|J^TJ|$, which requires $\dim(\z)$ Jacobian-vector products or an iterative algorithm \citep{caterini_rectangular_2021}.  We avoid this problem by setting $\alpha=1$ in \cref{eq:PF lagrangian}.  This yields the iPF objective:

\begin{align}\label{eq:iPF objective}
    \mathrm{L}_\text{iPF}(\theta) &= \sum_{\x \in \mathcal{D}} -\log p_\x(\x;\theta) + \Ip(\x;\theta) \nonumber \\
    &= \sum_{\z = g(\x),\x \in \mathcal{D}} -\log p_\z(\z) + \frac{1}{2}\sum_{\setstyle{k}\in \mathcal{P}}\log|\Jk{k}^T\Jk{k}|
\end{align}
The iPF objective is a novel lower bound on the log likelihood of a dataset that lies on a manifold.  Clearly the bound is tight when $\Ip(\x;\theta)=0$, so the learned model must trade off how close its contours are to principal manifolds with how well it represents data - both of which are desirable properties to have in a generative model.  The computational bottleneck of $\mathrm{L}_\text{iPF}(\theta)$ is the $\log|\Jk{k}^T\Jk{k}|$ terms, which each require $|\setstyle{k}|$  Jacobian-vector products to compute.  However, if $|\setstyle{k}| \ll \dim(\z)$, then $\mathrm{L}_\text{iPF}(\theta)$ is much more efficient to estimate than $\mathrm{L}(\theta)$ (see the next paragraph on unbiased estimates).  
\cref{eq:iPF objective} on its own cannot be used for training because there are no guarantees that training data will satisfy the condition $\x=f(\z)$.  Instead, we plug $\mathrm{L}_\text{iPF}(\theta)$ into the algorithm described in section 4 of \citep{caterini_rectangular_2021}.  This algorithm projects training data onto the generative manifold and maximizes the likelihood of the projected data, while also minimizing the reconstruction error.  See appendix \cref{appendix:mnist iPF} for a full description.

\paragraph{Unbiased estimates of the objectives}\label{sec:unbiased esimate}\mbox{} \\
In practice, we implement \cref{eq:PF objective} and \cref{eq:iPF objective} by randomly selecting $\setstyle{k}\in \mathcal{P}$, constructing $|\setstyle{k}|$ one-hot vectors where each vector has a single 1 at an index in $\setstyle{k}$, and evaluating each in a vector-Jacobian product (vjp) at $g(\x)$ or Jacobian-vector product (jvp) with $f(\z)$.  If each $\zk{k}$ is 1 dimensional, then the PF objective only requires a single vjp or jvp.  This means that the cost of training a PF is only slightly more expensive than training a regular normalizing flow and the cost of training an iPF is much more efficient than training an injective normalizing flow.  We provide Python code in appendix \cref{appendix:code}.

% \subsection{Working with PFs}
% \begin{itemize}
%     \item How to order the principal manifolds by looking at $\frac{1}{2}\log|\Jk{k}^T\Jk{k}|$
%     \item Implications for learning manifold densities and true manifold dimensionality
% \end{itemize}

% PFs have many nice theoretical properties, but can be difficult to train and interpret in practice.  We find that the constraint $\Ip=0$ can only be satisfied with normalizing flows that are very expressive.  We conjecture that the reason is because the constraint $\Ip=0$ requires that each $O(2^{|\mathcal{P}|})$ possible contour that can be constructed are orthogonal to the other contours.  As a result, we find it necessary to keep $\mathcal{P}$ small by using iPFs to model high dimensional data, increasing the size of each partition or using a feature extractor flow that can transform data into a simpler form for the PF.  Furthermore, the latent space of PFs is not trivial to interpret.  While it is true that the latent variables of PFs correspond to different principal manifolds, the index of the dimension corresponding to different principal manifolds can change (see \cref{fig:contour2d} for clear examples of this).  This means that a smooth path through over a principal manifold may require a discontinuous path through the latent space.  Nevertheless, we are able to identify which latent dimension corresponds to which principal manifold by comparing the values of $\frac{1}{2}\log|\Jk{k}^T\Jk{k}|$.

% USED TO BE A PARA BREAK HERE
\begin{figure*}[th!]
    \centering
    \includegraphics[width=0.9\textwidth]{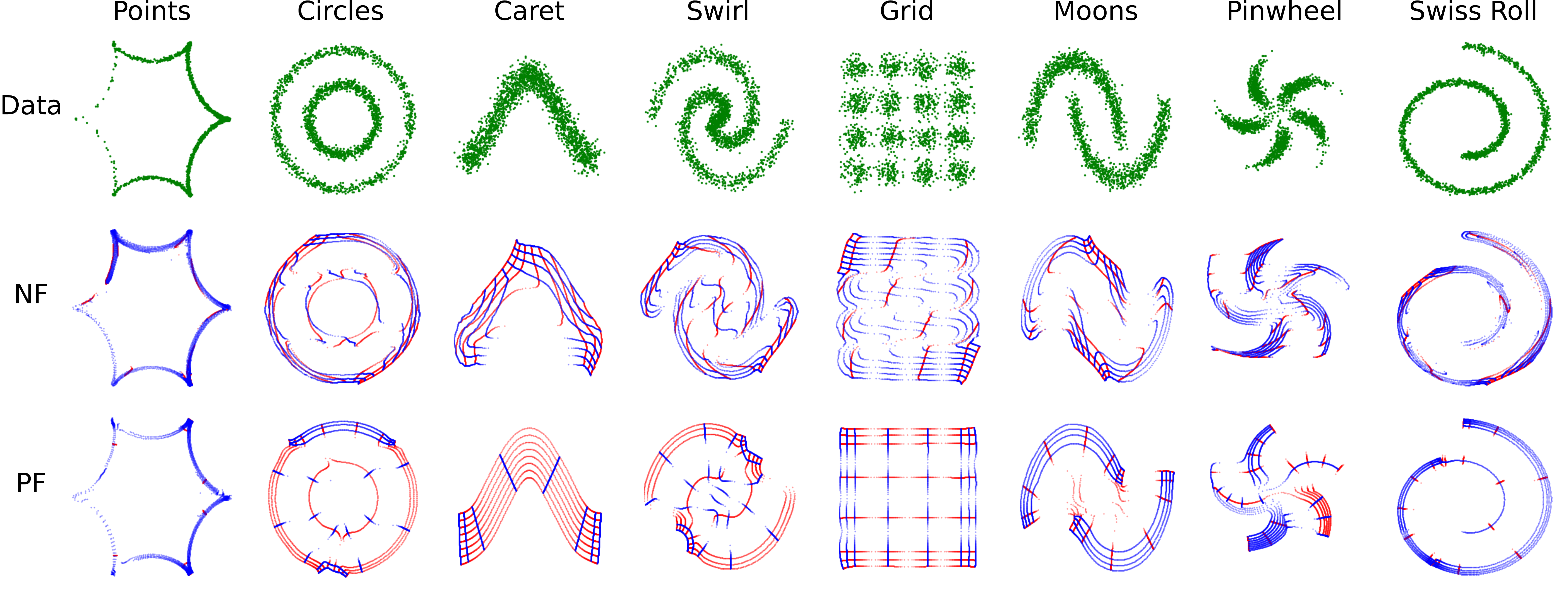}
    \caption{Contours for various synthetic datasets from a normalizing flow (NF) and principal manifold flow (PF).  Both flows learned to produce the correct samples (see \cref{appendix:2d experiments}) but only the PF learns the data's structure.}
    \label{fig:contour2d}
\end{figure*}
% USED TO BE A PARA BREAK HERE
\begin{table*}[]
    \centering
\begin{tabular}{llrrrrrrrr}
% \toprule
  &    &  Points &  Circles &  Caret &  Swirl &  Grid &  Moons &  Pinwheel &  Swiss Roll \\
\midrule
$\log p(x) (\mathbf{\uparrow})$ & NF &   -1.60 &    -3.10 &  -1.89 &  -0.19 & -6.02 &  -0.64 &     -3.28 &       -4.67 \\
  & PF &   -1.62 &    -3.12 &  -1.89 &  -0.20 & -6.02 &  -0.66 &     -3.29 &       -4.68 \\
\midrule
$\Ip (\mathbf{\downarrow})$ & NF &    1.60 &     1.18 &   0.61 &   0.71 &  0.39 &   0.64 &      0.77 &        1.38 \\
  & PF &    0.00 &     0.00 &   0.00 &   0.00 &  0.00 &   0.00 &      0.00 &        0.00 \\
% \bottomrule
\end{tabular}
    \caption{Numerical results for learning synthetic datasets.  The PF obtains a similar test set log likelihood to that of the normalizing flow (NF), but only the PF has small pointwise mutual information ($\Ip$).  Small values of $\Ip$ result in the orthogonal contours shown in \cref{fig:contour2d}.}
    \label{tab:2d results}
\end{table*}

\section{Related Work}
Our work plugs a methodological gap in the normalizing flows \citep{papamakarios_normalizing_2019,pmlr-v37-rezende15} related to finding structure within flows.  Although this is not crucial for applications such as density estimation or Neural-transport MCMC \citep{neutra}, the success of approaches in other deep generative models for finding low dimensional structure such as the $\beta$-VAE \citep{higgins_beta-vae_2017,pmlr-v80-alemi18a} and Style GAN \citep{karras2019style} are motivation to find structure in normalizing flows.  
% USED TO BE A PARA BREAK HERE
A subarea of flows research focuses on learning densities on manifolds.  \citep{gemici_normalizing_2016} introduced a proof of concept for learning a density over a specified manifold and since then other methods have extended the idea to other kinds of manifolds such as toris, spheres and  hyperbolic spaces \citep{rezende_normalizing_2020,bose2020latent}.  A related class of flows are dedicated to both learning manifolds and densities over them \citep{pmlr-v108-kumar20a,brehmer_flows_2020,kalatzis_multi-chart_2021,caterini_rectangular_2021,kothari_trumpets_2021}.  Our work is different because we focus on flows with density in the full data space and we do not focus on learning any single manifold.  Additionally, there has been work in flows aimed at constructing architectures so that structure can emerge during training \citep{zhang2021on,cunningham_normalizing_2020,pmlr-v130-cunningham21a}, however these methods have no guarantees that they will recover the intended structure whereas PFs do.

There are other works that impose orthogonality conditions on Jacobian matrices.  Conformal embedding flows \citep{ross2021conformal} constructs an injective flow that has an orthogonal times a scalar Jacobian matrix to learn densities over manifolds.  Our Jacobian structure is more flexible because it only requires $J^TJ$ to be block diagonal.  We also note that our method can be used to learn conformal mappings if the regularizer is used on the Jacobian and its transpose.  \citep{dombrowski2021diffeomorphic} presents a way to apply flows that have learned the structure of a dataset to generating counterfactuals by using optimization in the latent space, which is shown to adhere to the flow's generative manifold in the data space.  \citet{wei2021orojar} and \citet{gropp_isometric_2020} propose regularizers to ensure that the Jacobian matrix of their models are orthogonal.  As mentioned earlier, our Jacobian structure is much more flexible.
% USED TO BE A PARA BREAK HERE
The most similar work to ours is independent mechanism analysis (IMA) \citep{gresele2021independent}. IMA is motivated by independent component analysis and causal inference while ours is motivated by uncovering the structure of data.  We introduce novel insights on the geometry of flows and the densities on their contours, orthogonality conditions for both injective flows and flows that are not easily invertible, and a test time algorithm for computing densities on manifolds.
PCA \citep{Jolliffe2011} and its nonlinear extensions \citep{Jolliffe2011,gorban_principal_2008} have the same goal as PFs of finding the principal structure of data.    \citet{cramer2021principal} treat PCA as a linear PF, but do not consider the nonlinear case. Work related to principal manifolds, such as locally linear embeddings \citep{ghojogh2021generative}, differ from ours primarily in that we use parametric functions to learn the geometry of data.

\section{Experiments}

% \begin{figure}
%     \centering
%     \includegraphics[width=1.0\linewidth]{figures/3d manifold/manifold_density.pdf}
%     \caption{PFs are the only class of flows that can learn densities on manifolds with variable dimensionality.  The dataset is generated on a manifold that is 1D near the origin and 2D elsewhere.  At test time, the density for each data point is computed using \cref{def:manifold corrected density}.}
%     \label{fig:manifold densities}
% \end{figure}
\begin{figure*}
    \centering
    \includegraphics[width=0.8\linewidth]{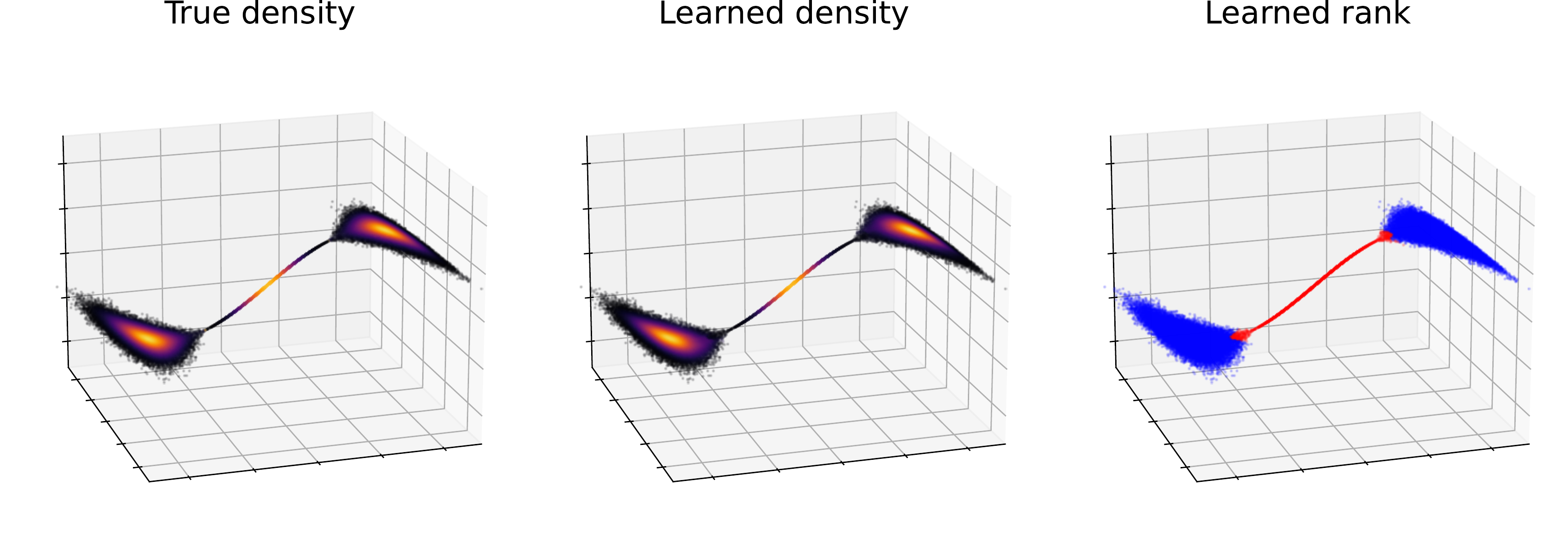}
    \caption{PFs are the only class of flows that can learn densities on manifolds with variable dimensionality.  The dataset is generated on a manifold that is 1D near the origin and 2D elsewhere.  At test time, the density for each data point is computed using \cref{def:manifold corrected density}.}
    \label{fig:manifold densities}
\end{figure*}

Our experiments showcase the capabilities of PFs to learn the principal manifolds of data, perform density estimation on data that is generated on a variable dimensional dataset, and learn high dimensional data embedded on a low dimensional manifold.  All of our experiments were written using the JAX \citep{jax2018github} Python library.  We provide extended results and details of our models in \cref{appendix:extended results}.

\subsection{2D Synthetic Datasets}
We trained standard normalizing flow and PF on various synthetic 2D datasets.  Both flows have an architecture with 10 coupling layers, each with a logistic mixture cdf with 8 components, logit and shift-scale transformer \citep{ho_flow_2019,papamakarios_normalizing_2019} and 5 layer residual network with 64 hidden units conditioner.  We applied a matrix vector product and act norm layer in between each coupling layer \citep{kingma_glow_2018}.  Note that logistic mixture cdfs require an iterative algorithm to invert.

The log likelihood and pointwise mutual information $(\Ip)$ of the test sets are shown in \cref{tab:2d results}.  We see from the likelihoods that the PF is able to learn the datasets as well as the standard flow while achieving a small value of $\Ip$.  The low $\Ip$ is reflected by the contours in \cref{fig:contour2d}.  In line with our theory, the contours of the PF are orthogonal to each other and are oriented in the directions of maximum variance.

% \begin{figure}[th!]
%     \centering
%     \includegraphics[width=0.47\textwidth]{figures/2d_manifold/2d_contours.pdf}
%     % \includegraphics[width=\textwidth]{figures/2d_manifold/2d_contours.png}
%     \caption{Contours for various synthetic datasets from a normalizing flow (NF) and principal manifold flow (PF).  Both flows learned to produce the correct samples (see \cref{appendix:2d experiments}) but only the PF learns the data's structure.}
%     \label{fig:contour2d}
% \end{figure}

\subsection{Learning manifold densities of varying rank}\label{sec:manifold 3d}

PFs have the unique ability to learn densities on manifolds with unknown rank.  All existing density estimation algorithms on manifolds using flows require specifying the dimensionality of the manifold beforehand, but PFs do not because they will automatically learn the underlying structure of the dataset.  The leftmost plot of \cref{fig:manifold densities} shows the target probability distribution whose samples lie on either a 1D or 2D manifold.  The data is generated by first sampling two univariate random variables $z_1$ and $z_2$ from a Gaussian mixture model and standard Gaussian respectively and then transforming $z=(z_1,z_2)$ to the data space with the equation $x=(z_1, z_2\text{max}(0,1-|\frac{1}{z_1}|), \sin(z_1))$.  Notice that $x$ is one dimensional when $|z_1| < 1$ and two dimensional otherwise.  During training we perturb the dataset with a small amount of Gaussian noise so that the training data has full rank.  See \cref{appendix:3d experiments} for a full description of the data and model and extended results.  We use the method described in \cref{def:manifold corrected density} to compute the rank and density of each data point in the test set.  We see from the center plot of \cref{fig:manifold densities} that the PF correctly recovers the densities of the test data samples.  The final forward KL divergence from the learned density and true density is 0.0146.  

\subsection{iPF}\label{section:iPF experiment}
Here we show that the iPF learning algorithm does in fact learn an injective flow with contours that are close to principal manifolds, and that the intuition about how contours relate to the principal manifolds does help explain the generative behavior of flows.  We trained an iPF and standard injective normalizing flow (iNF) on the MNIST dataset \citep{mnist}.  The iPF and iNF both had the same architecture consisting of 20 layers of GLOW \citep{kingma_glow_2018}, a slice layer that removes all but 10 of the latent dimensions (so that the latent space is 10 dimensional), and then another 10 layers of neural spline flows \citep{durkan_neural_2019}.  See appendix \cref{appendix:mnist iPF} for details on the model and the training.  Note that the iPF required roughly 10 times less resources to train because we computed a single jvp to estimate \cref{eq:iPF objective} while the iNF required 10 jvps to compute \cref{eq:change of variables general}.

\cref{fig:iPF contours} shows a similarity plot between sorted contours of each model and the true principal components.  The columns represent the principal components sorted by eigenvalue while the rows represent the tangents of the contours (columns of $J$) in increasing order of the diagonal of $J^TJ$.  The intensity of each cell is the average absolute value of the cosine similarity between $J$ and a principal component.  The plot of iPF is highlighted along the diagonal, which indicates that the contours are mostly aligned with the principal components whereas the plot for the iNF is highlighted along the last column, which indicates that the contours are mostly aligned with only the largest principal component.

\cref{fig:contour traversal 1} and \cref{fig:contour traversal 2} show a traversal of the largest and 5th largest contours of the iPF and iNF respectively.  We moved along the contours by computing the Jacobian matrix of the flow at the current $\z$, ordering the contours according to the diagonal of $J^TJ$, and then taking a step of $0.02$ on the dimension of $\z$ corresponding to the contour we want to traverse.  We took 500 of these steps and displayed every $50^\text{th}$ image in the figures.  The images generated on the top contours for both models are varied as expected.  The images on the 5th largest contours of the iPF are only varied slightly, which matches the results from \cref{fig:iPF contours} that the 5th largest contours will be oriented similarly to the 5th largest principal manifold and should therefore result in only a minimal amount of change.  The iNF, on the other hand, generates images on the 5th largest contour that are similar to those generated on the largest contour.  This also matches the intuition from \cref{fig:iPF contours} that the contours of the iNF are mostly aligned with the largest principal manifold.

\begin{figure}
\centering
\sbox{\measurebox}{%
  \begin{minipage}[b]{.45\textwidth}
  \subfloat
    []
    {\label{fig:iPF contours}\includegraphics[width=\textwidth]{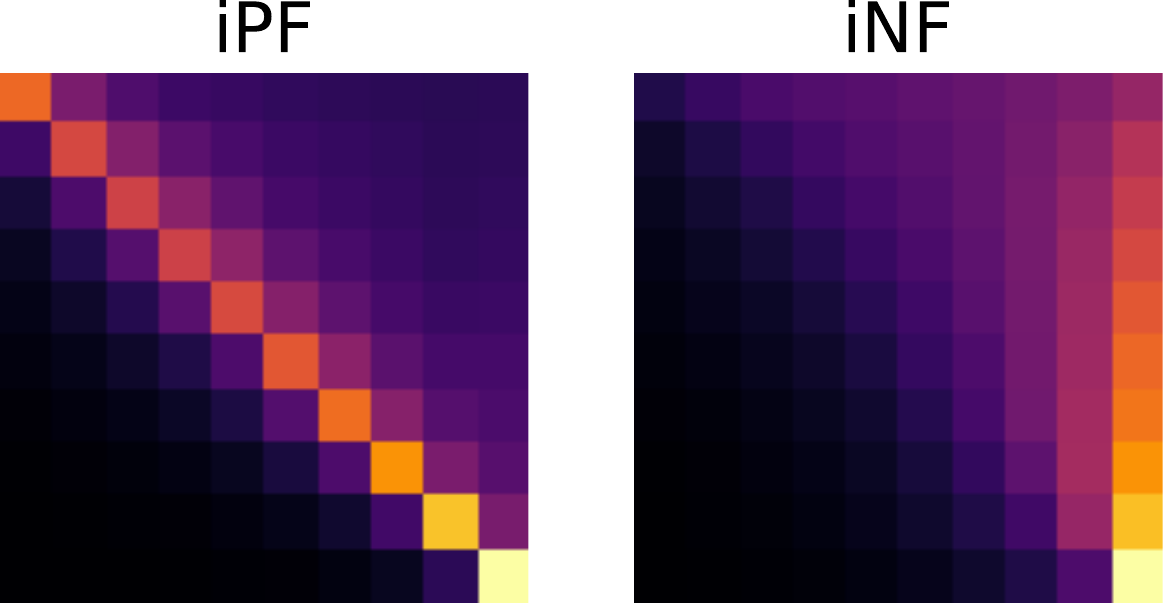}}
  \end{minipage}}
\usebox{\measurebox}\qquad
\begin{minipage}[b][\ht\measurebox][s]{.45\textwidth}\label{fig:contour traversal}
\centering
\subfloat
  []
  {\label{fig:contour traversal 1}\includegraphics[width=\textwidth]{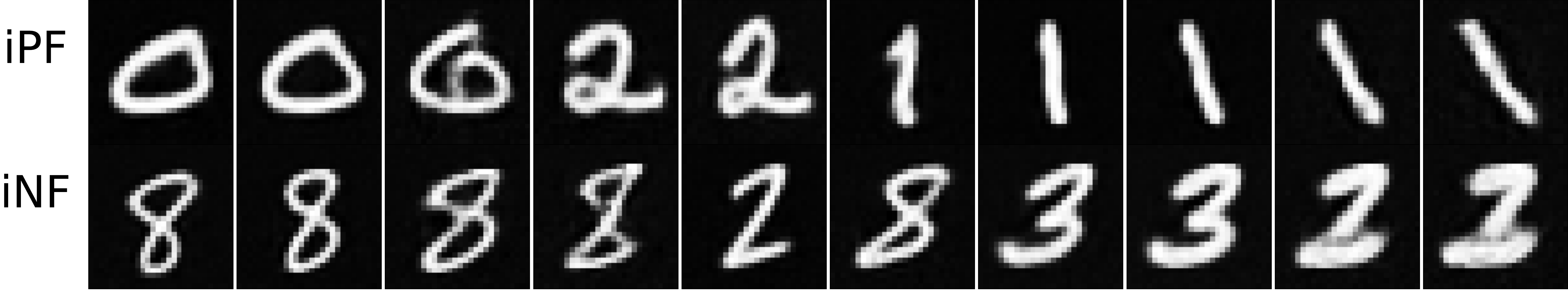}}

\vfill

\subfloat
  []
  {\label{fig:contour traversal 2}\includegraphics[width=\textwidth]{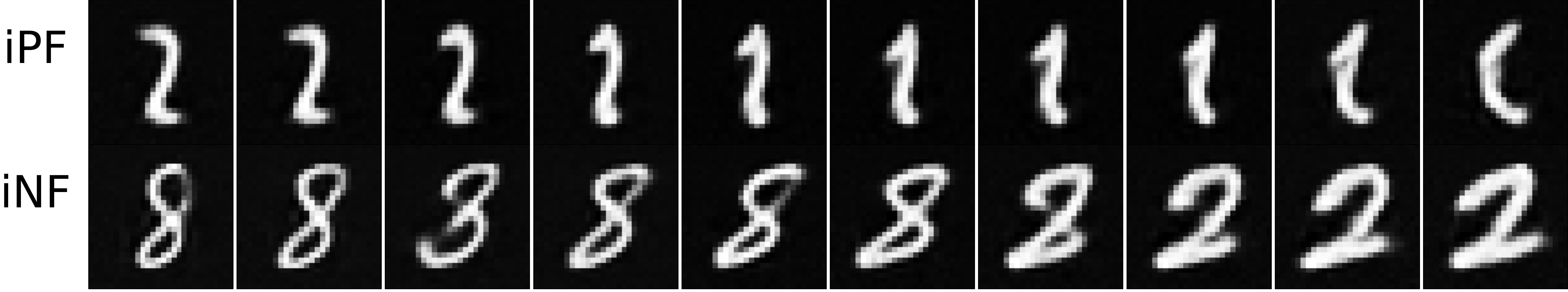}}
\end{minipage}
\caption{Similarity plot (\cref{fig:iPF contours}) between sorted contours of each model and the true principal components and traversal of the largest (\cref{fig:contour traversal 1}) and 5th largest (\cref{fig:contour traversal 2}) contours of the iPF and iNF.  See \cref{section:iPF experiment} for more details.}
\label{fig:iPF experiment}
\end{figure}

% \begin{figure}
%     \centering
%     \includegraphics[width=0.4\linewidth]{figures/iPCF/principal_components.pdf}
%     \caption{Similarity plot between sorted contours of each model and the true principal components.  The iPF has contours that are mostly aligned with the principal components while the iNF has contours that are aligned with mainly the largest principal component.  See \cref{section:iPF experiment} for details.}
%     \label{fig:iPF contours}
% \end{figure}

% \begin{figure}
%     \centering
%     \begin{subfigure}{0.5\textwidth}
%         \centering
%         \includegraphics[width=0.93\linewidth]{figures/iPCF/principal_manifold_traversal.pdf}
%     \end{subfigure}
%     % \newline
%     \begin{subfigure}{0.5\textwidth}
%         \centering
%         \includegraphics[width=0.93\linewidth]{figures/iPCF/principal_manifold_traversal_5th_largest.pdf}
%     \end{subfigure}
%     \caption{Traversal of the largest (top row) and 5th largest (bottom row) contours of the iPF and iNF.  The variation of the images generated by the iPF is greater on the largest contour than 5th largest contour, which supports the results from \cref{fig:iPF contours} that the contours are roughly the principal manifolds, while the variation of the images from the iNF are similar for the two contours.}
%     \label{fig:contour traversal}
% \end{figure}

\section{Conclusion}
We introduced principal manifold flows, a type of normalizing flow whose latent variables generate its principal manifolds.  We investigated the generative behavior of flows using principal manifolds and contours to understand how a flow assign probability density to its samples.  This analysis helped us define PFs and develop an efficient general purpose learning algorithm.  Furthermore, we found an objective function to train injective PFs that avoided the need to compute a difficult Jacobian determinant during training.  We showed how to interpret the contours of PFs and proposed a simple test to match a contour with a principal manifold.  This test was then shown to help perform density estimation on the true data manifold at test time.  Our experiments demonstrated the PFs are effective tools for learning the principal manifolds of low dimensional data, or high dimensional data that is embedded on a low dimensional manifold, and that PFs are capable of performing density estimation on data that is generated on a variable dimensional manifold. 

% added for approval routing - SJ
\subsubsection*{Acknowledgments}
This material is based upon work supported by 
U.S. Army Research Laboratory Cooperative Research Agreement W911NF-17-2-0196,
U.S. National Science Foundation(NSF) grants \#1740079, and the United States Air Force and DARPA under Contract No. FA8750-20-C-0002.
The views, opinions
and/or findings expressed are those of the author(s) and should not be interpreted as representing
the official views or policies of the Department of Defense or the U.S. Government.

\clearpage

\bibliography{bibliography.bib}
\bibliographystyle{unsrtnat}

\appendix

% \section{Limitations}\label{appendix:practical considerations}
% PFs have many nice theoretical properties, but can be difficult to train and interpret in practice.  We find that the constraint $\Ip=0$ can only be satisfied with normalizing flows that are very expressive.  We conjecture that the reason is because the constraint $\Ip=0$ requires that each $O(2^{|\mathcal{P}|})$ possible contour that can be constructed are orthogonal to the other contours.  As a result, we find it necessary to keep $\mathcal{P}$ small by using iPFs to model high dimensional data, increasing the size of each partition or using a feature extractor flow that can transform data into a simpler form for the PF.  Furthermore, the latent space of PFs is not trivial to interpret.  While it is true that the latent variables of PFs correspond to different principal manifolds, the index of the dimension corresponding to different principal manifolds can change (see \cref{fig:contour2d} for clear examples of this).  This means that a smooth path through over a principal manifold may require a discontinuous path through the latent space.

\section{Python implementation}\label{appendix:code}
Below are Python implementations of the PF objective function.  The code uses the JAX \citep{jax2018github} Python library.

% \begin{minted}{python}
\begin{python}
import jax
import jax.numpy as jnp
import jax.scipy.stats.multivariate_normal as gaussian
import einops
from jax.random import randint

def PF_objective_brute_force(flow, x, P, alpha=5.0):
  """ Brute force implementation of the PF objective.
      Implemented for unbatched 1d inputs for simplicity

  Inputs:
    flow  - Function that accepts an unbatched 1d input
            and returns a 1d output and the log determinant
    x     - Unbatched 1d input
    P     - List of numpy arrays that form a partition
            over range(x.size)
    alpha - Regularization hyperparameter

  Outputs:
     objective - PFs objective
  """
  # Evaluate log p(x) with a Gaussian prior
  z, log_det = flow(x)
  log_pz = gaussian.logpdf(z, 0.0, 1.0).sum()
  log_px = log_pz + log_det

  # Create the Jacobian matrix for every item in the batch
  G = jax.jacobian(lambda x: flow(x)[0])(x)

  # Compute Ihat_P
  Ihat_P = -log_det
  for k in P:
    Gk = G[k,:]
    Ihat_P += 0.5*jnp.linalg.slogdet(Gk@Gk.T)[1]

  objective = -log_px + alpha*Ihat_P
  return objective.mean()
  
def PF_objective_unbiased(flow, x, rng_key, alpha=5.0):
  """ Unbiased estimate of the PF objective when the partition size is 1

  Inputs:
    flow    - Function that accepts an unbatched 1d input
              and returns a 1d output and the log determinant
    x       - Unbatched 1d input
    rng_key - JAX random key
    alpha   - Regularization hyperparameter

  Outputs:
     objective - PFs objective
  """
  # Evaluate log p(x) with a Gaussian prior and construct the vjp function
  z, vjp, log_det = jax.vjp(flow, x, has_aux=True)
  log_pz = gaussian.logpdf(z, 0.0, 1.0).sum()
  log_px = log_pz + log_det

  # Sample an index in the partition
  z_dim = z.shape[-1]
  k = random.randint(rng_key, minval=0, maxval=z_dim, shape=(1,))
  k_onehot = (jnp.arange(z_dim) == k).astype(z.dtype)
  
  # Evaluate the k'th row of G and compute an unbiased estimate of Ihat_P
  Gk, = vjp(k_onehot)
  GkGkT = (Gk**2).sum()
  Ihat_P = -log_det + z_dim*0.5*jnp.log(GkGkT)

  objective = -log_px + alpha*Ihat_P
  return objective.mean()

def iPF_objective_unbiased(flow, x, rng_key, gamma=10.0):
  """ Unbiased estimate of the iPF objective when the partition size is 1

  Inputs:
    flow    - Function that accepts an unbatched 1d input
              and returns a 1d output and the log determinant
    x       - Unbatched 1d input
    rng_key - JAX random key
    gamma   - Regularization hyperparameter

  Outputs:
     objective - iPFs objective
  """
  # Pass x through to the latent space and compute the prior
  z, _ = flow(x)
  log_pz = gaussian.logpdf(z, 0.0, 1.0).sum()
  
  # Sample an index in the partition
  z_dim = z.shape[-1]
  k = random.randint(rng_key, minval=0, maxval=z_dim, shape=(1,))
  k_onehot = (jnp.arange(z_dim) == k).astype(z.dtype)
  
  # Compute the reconstruction and k'th row of J
  x_reconstr, Jk = jax.jvp(lambda x: flow(x, inverse=True)[0], (z,), (k_onehot,))
  JkTJk = (Jk**2).sum()
  reconstruction_error = jnp.sum((x - x_reconstr)**2)
  
  # Compute the objective function
  objective = -log_pz + 0.5*jnp.log(JkTJk) + gamma*reconstruction_error
  return objective.mean()

def construct_partition_mask(index, z_shape):
  """ In general we can find the i'th row of a matrix A
      by computing A.T@mask where mask is zeros everywhere
      except at the i'th index where it is 1.

      This function finds all of the masks needed to find
      the rows in G that are in the index'th partition.

  Inputs:
    index   - Batched array of integers
    z_shape - Shape of the latent variable

  Outputs:
     masks - Array of 0s and 1s that will be used to
             find the rows of G within the index'th partition.
  """
  batch_size, H, W, C = z_shape
  n_partitions = C

  # The only non zero element of i'th row of
  # partition_mask is at the index[i]'th position
  # This is used to select a partition.
  # shape is (batch_size, C)
  partition_mask = jnp.arange(n_partitions) == index[:,None]

  # Create masks that will let us find the k'th rows of G using masked vjps.
  partition_size = H*W
  G_selection_mask = jnp.eye(partition_size)
  G_selection_mask = G_selection_mask.reshape((partition_size, H, W))

  # Put the masks together
  masks = jnp.einsum("bc,phw->pbhwc", partition_mask, G_selection_mask)
  return masks

def unbiased_objective_image(flow, x, rng_key, alpha=5.0, vectorized=True):
  """ PFs objective function for images.  Number of partitions is given
      by number of channels of output.

  Inputs:
    flow      - Function that accepts an batched 3d input
                and returns a batched 3d output and the log determinant
    x         - Batched 3d input with channel on last axis
    rng_key   - JAX random key
    alpha     - Regularization hyperparameter
    vectorize - Should all of the vjps be evaluated in parallel?

  Outputs:
     objective - PFs objective for images
  """
  # Assume that we partition over the last axis of z
  # and that x is a batched image with channel on the last axis
  batch_size, H, W, C = x.shape

  # Evaluate log p(x) and retrieve the function that
  # lets us evaluate vector-Jacobian products
  z, _vjp, log_det = jax.vjp(flow, x, has_aux=True)
  vjp = lambda v: _vjp(v)[0] # JAX convention to return a tuple
  log_pz = gaussian.logpdf(z, 0.0, 1.0).sum(axis=range(1, z.ndim))
  log_px = log_pz + log_det

  # Randomly sample the index of the partition we will evaluate
  n_partitions = z.shape[-1]
  index = randint(rng_key, minval=0, maxval=n_partitions, shape=(batch_size,))

  # Construct the masks that we'll use to find the index'th partition of G.
  # masks.shape == (partition_size, batch_size, H, W, C)
  masks = construct_partition_mask(index, z.shape)

  # Evaluate the vjp each of the n_partition masks
  if vectorized:
    # This is memory intensive but fast
    Gk = jax.vmap(vjp)(masks)
  else:
    # This is slow but memory efficient
    Gk = jax.lax.map(vjp, masks)

  # Each element of GG^T is the dot product between rows of G
  # Construct GkGk^T and then take its log determinant
  Gk = einops.rearrange(Gk, "p b H W C -> b p (H W C)")
  GkGkT = jnp.einsum("bij,bkj->bik", Gk, Gk)
  Ihat_P = 0.5*jnp.linalg.slogdet(GkGkT)[1]*n_partitions - log_det

  objective = -log_px + alpha*Ihat_P
  return objective.mean()
\end{python}
% \end{minted}

\newpage

\section{Contour Cookbook}\label{appendix:contour cookbook}
Below we list properties of contour densities and pointwise mutual information and their inverse variants.  Recall from our assumptions stated in the main text that a normalizing flow generates samples under the model $\z \sim p_\z(\z)=\prod_{\setstyle{k}\in \mathcal{P}}p_\smallsetstyle{k}(\z_\smallsetstyle{k}), \quad \x=f(\z)$ where $\dim(\x) \geq \dim(\z)$.  $f(\z)$ has the inverse $\z=f^{-1}(\x)=g(\x)$ and Jacobian matrix $J=\frac{df(\z)}{d\z}$ while the Jacobian matrix of the inverse function is $G=\frac{dg(\x)}{d\x}$.  $\Jk{k}$ is the matrix whose columns are the columns of $J$ with indices in $\setstyle{k}$ and $\Gk{k}$ is the matrix whose rows are the rows of $G$ with indices in $\setstyle{k}$.  

Below we assume that $\setstyle{s}$ and $\setstyle{t}$ are disjoint subsets of the integers in $[1,\dots,\dim(\z)]$ and that the indices of $\setstyle{s}$ and $\setstyle{t}$ are ordered so that results with block matrices can be presented clearly.  This is a valid assumption because the latent dimension can always be renumbered.

\subsection{Definitions}
\begin{enumerate}
    \item $\Lk{k} \overset{{\scriptstyle \Delta}}{=} \log p_{\setstyle{k}}(\zk{k}) -\frac{1}{2} \log|\Jk{k}^T \Jk{k}|$ \label{def:1}
    \item $\Lhatk{k} \overset{{\scriptstyle \Delta}}{=} \log p_{\setstyle{k}}(\zk{k}) + \frac{1}{2} \log|\Gk{k} \Gk{k}^T|$ \label{def:2}
    \item $\mathcal{L} \overset{{\scriptstyle \Delta}}{=} \log p_\x(\x) = \log p_\z(\z) - \frac{1}{2}|J^TJ|$ \label{def:3}
    \item $\widehat{\mathcal{L}} \overset{{\scriptstyle \Delta}}{=} \log p_\z(\z) + \frac{1}{2}|GG^T|$ \label{def:3.5}
    \item $\pmi{s}{t} \overset{{\scriptstyle \Delta}}{=} \Lk{s+t} - \Lk{s} - \Lk{t}$ \label{def:4}
    \item $\pmihat{s}{t} \overset{{\scriptstyle \Delta}}{=} \Lhatk{s+t} - \Lhatk{s} - \Lhatk{t}$ \label{def:5}
    \item $\Ip \overset{{\scriptstyle \Delta}}{=} \mathcal{L}  - \sum_{\setstyle{k}\in \mathcal{P}}\Lk{k}$ \label{def:6}
    \item $\Ihatp \overset{{\scriptstyle \Delta}}{=} \widehat{\mathcal{L}}  - \sum_{\setstyle{k}\in \mathcal{P}}\Lhatk{k}$ \label{def:7}
\end{enumerate}
\subsection{Claims}
\begin{enumerate}
    \item $\pmi{s}{t} = -\frac{1}{2}\log|\Jk{s+t}^T\Jk{s+t}| + \frac{1}{2}\log|\Jk{s}^T\Jk{s}| + \frac{1}{2}\log|\Jk{t}^T\Jk{t}|$ \label{claim:1}
    \item $\Ip = -\frac{1}{2}\log|J^TJ| + \frac{1}{2}\sum_{\setstyle{k}\in \mathcal{P}}\log|\Jk{k}^T\Jk{k}|$ \label{claim:1.5}
    \item $\pmi{s}{t} = -\frac{1}{2}\log|I - \Jk{s}^\parallel \Jk{t}^\parallel|$ where $A^\parallel = A(A^TA)^{-1}A^T$ denotes the projection matrix of $A$. \label{claim:2}
    \item $\pmi{s}{t} \geq 0$ \label{claim:3}
    \item $\Ip \geq 0$ \label{claim:3.5}
    \item $\pmi{s}{t} = 0 \text{ if and only if } \Jk{s+t} = \Ak{U}{s+t}^\parallel\Ak{\Sigma}{s+t}\begin{bmatrix}\Ak{V}{s}^T & 0 \\ 0 & \Ak{V}{t}^T\end{bmatrix}$ where $\Ak{U}{s+t}^\parallel$ is semi-orthogonal, $\Ak{V}{s}$ and $\Ak{V}{t}$ are orthogonal and $\Ak{\Sigma}{s+t}$ is diagonal. \label{claim:4}
    \item $\Ip = 0 \text{ if and only if }J = U^\parallel\Sigma\begin{bmatrix}\APk{V}{1}^T & 0 & 0 & 0 \\ 0 & \APk{V}{2}^T & 0 & 0 \\  0 & 0 & \ddots & \vdots \\  0 & 0 & \dots & \APk{V}{|\mathcal{P}|}^T\end{bmatrix}$ where $U^\parallel$ is a semi orthogonal matrix, $\Sigma$ is a diagonal matrix and each $\APk{V}{k}^T$ is an orthogonal matrix with same number of rows and columns as the $k'th$ element of $\mathcal{P}$. \label{claim:4.5}
    \item $\pmi{s}{t} = 0$ if and only if $\fcontour{s}$ and $\fcontour{t}$ intersect orthogonally. \label{claim:5}
    \item $\pmihat{s}{t} =  \frac{1}{2}\log|\Gk{s+t}\Gk{s+t}^T| - \frac{1}{2}\log|\Gk{s}\Gk{s}^T| - \frac{1}{2}\log|\Gk{t}\Gk{t}^T|$ \label{claim:6}
    \item $\Ihatp = \frac{1}{2}\log|GG^T| - \frac{1}{2}\sum_{\setstyle{k}\in \mathcal{P}}\log|\Gk{k}\Gk{k}^T|$ \label{claim:6.5}
    \item $\pmihat{s}{t} = \frac{1}{2}\log|I - \Gk{s}^\parallel \Gk{t}^\parallel|$ \label{claim:7}
    \item $\pmihat{s}{t} \leq 0$ \label{claim:8}
    \item $\Ihatp \leq 0$ \label{claim:8.5}
    \item $\pmihat{s}{t} = 0 \text{ if and only if } \Gk{s+t} = \begin{bmatrix}\Ak{V}{s} & 0 \\ 0 & \Ak{V}{t}\end{bmatrix}\Ak{\Sigma}{s+t}{\Ak{U}{s+t}^\parallel}^T$ where ${\Ak{U}{s+t}^\parallel}^T$ is semi-orthogonal, $\Ak{V}{s}$ and $\Ak{V}{t}$ are orthogonal and $\Ak{\Sigma}{s+t}$ is diagonal. \label{claim:9}
    \item $\Ihatp = 0 \text{ if and only if }J = \begin{bmatrix}\APk{V}{1} & 0 & 0 & 0 \\ 0 & \APk{V}{2} & 0 & 0 \\  0 & 0 & \ddots & \vdots \\  0 & 0 & \dots & \APk{V}{|\mathcal{P}|}\end{bmatrix}\Sigma{U^\parallel}^T$ where ${U^\parallel}^T$ is a semi orthogonal matrix, $\Sigma$ is a diagonal matrix and each $\APk{V}{k}$ is an orthogonal matrix with same number of rows and columns as the $k'th$ element of $\mathcal{P}$. \label{claim:10}    
    \item If $\dim(\x)=\dim(\z)$, then $\Ip = 0$ if and only if $\Ihatp = 0$ \label{claim:11}
\end{enumerate}

\subsection{Proofs}
\paragraph{Proof of claim \ref{claim:1}}\label{proof:1}
\textit{$\pmi{s}{t} = - \frac{1}{2}\log|\Jk{s+t}^T\Jk{s+t}| + \frac{1}{2}\log|\Jk{s}^T\Jk{s}| + \frac{1}{2}\log|\Jk{t}^T\Jk{t}|$}
\begin{proof}
\begin{align}
    \pmi{s}{t} &= \Lk{s+t} - \Lk{s} - \Lk{t} \\
    &= \underbrace{\log \frac{p_{\setstyle{s}+\setstyle{t}}(\zk{s+t})}{p_{\setstyle{s}}(\zk{s})p_{\setstyle{t}}(\zk{t})}}_{=0\text{ by assumption of how prior factors}} - \frac{1}{2}\log|\Jk{s+t}^T\Jk{s+t}| + \frac{1}{2}\log|\Jk{s}^T\Jk{s}| + \frac{1}{2}\log|\Jk{t}^T\Jk{t}|\\
    &=  - \frac{1}{2}\log|\Jk{s+t}^T\Jk{s+t}| + \frac{1}{2}\log|\Jk{s}^T\Jk{s}| + \frac{1}{2}\log|\Jk{t}^T\Jk{t}|
\end{align}
\end{proof}

\paragraph{Proof of claim \ref{claim:1.5}}\label{proof:1.5}
\textit{$\Ip = -\frac{1}{2}\log|J^TJ| + \frac{1}{2}\sum_{\setstyle{k}\in \mathcal{P}}\log|\Jk{k}^T\Jk{k}|$}
\begin{proof}
\begin{align}
    \Ip &= \mathcal{L} - \sum_{\setstyle{k}\in \mathcal{P}}\Lk{k} \\
    &= \underbrace{\log \frac{p_\z(\z)}{\prod_{\setstyle{k}\in \mathcal{P}}p_\smallsetstyle{k}(\z_\smallsetstyle{k})}}_{=0\text{ by assumption of how prior factors}} - \frac{1}{2}\log|J^TJ| + \frac{1}{2}\sum_{\setstyle{k}\in \mathcal{P}}\log|\Jk{k}^T\Jk{k}|\\
    &= -\frac{1}{2}\log|J^TJ| + \frac{1}{2}\sum_{\setstyle{k}\in \mathcal{P}}\log|\Jk{k}^T\Jk{k}|
\end{align}
\end{proof}

\paragraph{Proof of claim \ref{claim:2}}\label{proof:2}
\textit{$\pmi{s}{t} = -\frac{1}{2}\log|I - \Jk{s}^\parallel \Jk{t}^\parallel|$ where $A^\parallel = A(A^TA)^{-1}A^T$ denotes the projection matrix of $A$.}
\begin{proof}
\begin{align}
    |\Jk{s+t}^T\Jk{s+t}| &= |\begin{bmatrix}\Jk{s}^T \\ \Jk{t}^T\end{bmatrix}\begin{bmatrix}\Jk{s} & \Jk{t}\end{bmatrix}| \\
    &= |\begin{bmatrix}\Jk{s}^T\Jk{s} & \Jk{s}^T\Jk{t} \\ \Jk{t}^T\Jk{s} & \Jk{t}^T\Jk{t}\end{bmatrix}| \\
    &= |\Jk{s}^T\Jk{s}||\Jk{t}^T\Jk{t} - \Jk{t}^T\underbrace{\Jk{s}(\Jk{s}^T\Jk{s})^{-1}\Jk{s}^T}_{\Jk{s}^\parallel}\Jk{t}| \label{eq:placeholder} \\
    &= |\Jk{s}^T\Jk{s}||\Jk{t}^T\Jk{t}||I - \Jk{s}^\parallel \underbrace{\Jk{t}(\Jk{t}^T\Jk{t})^{-1}\Jk{t}^T}_{\Jk{t}^\parallel}| \\
    &= |\Jk{s}^T\Jk{s}||\Jk{t}^T\Jk{t}||I - \Jk{s}^\parallel\Jk{t}^\parallel|
\end{align}
Therefore $\frac{1}{2}\log|\Jk{s}^T\Jk{s}| + \frac{1}{2}\log|\Jk{t}^T\Jk{t}| - \frac{1}{2}\log|\Jk{s+t}^T\Jk{s+t}| = -\frac{1}{2}\log|I - \Jk{s}^\parallel\Jk{t}^\parallel|$.  An application of claim \ref{claim:1} completes the proof that $\pmi{s}{t} = -\frac{1}{2}\log|I - \Jk{s}^\parallel \Jk{t}^\parallel|$.
\end{proof}

\paragraph{Proof of claim \ref{claim:3}}\label{proof:3}
\textit{$\pmi{s}{t} \geq 0$}
\begin{proof}
First we will show that $I - \Jk{s}^\parallel \Jk{t}^\parallel$ is a positive semi-definite matrix.  Let $x$ be some vector.
\begin{align}
    x^T(I - \Jk{s}^\parallel \Jk{t}^\parallel)x &= x^Tx - x^T\Jk{s}^\parallel \Jk{t}^\parallel x \\
    &= |x|_2^2(1 - \underbrace{\frac{x}{|x|_2}}_{\hat{x}}\Jk{s}^\parallel \Jk{t}^\parallel\frac{x}{|x|_2})
\end{align}
$\Jk{s}^\parallel$ and $\Jk{t}^\parallel$ are orthogonal projection matrices, so their operator norm is less than or equal to 1.  By definition of the operator norm, we have that $||\Jk{s}^\parallel \hat{x}||_{\text{op}} \leq 1$ and $||\Jk{t}^\parallel \hat{x}||_{\text{op}} \leq 1$.  It follows that $\hat{x}\Jk{s}^\parallel \Jk{t}^\parallel\hat{x} \leq ||\Jk{s}^\parallel \hat{x}||_{\text{op}}||\Jk{t}^\parallel \hat{x}||_{\text{op}} \leq 1$.  So
\begin{align}
    |x|_2^2(1 - \hat{x}\Jk{s}^\parallel \Jk{t}^\parallel\hat{x}) \geq |x|_2 \geq 0
\end{align}
It is known that if $A$ is positive semi-definite, then $\log |A| \leq \Tr(A - I)$.  We can now apply this bound to $I - \Jk{s}^\parallel \Jk{t}^\parallel$:
\begin{align}
    -\frac{1}{2}\log|I - \Jk{s}^\parallel \Jk{t}^\parallel| &\geq -\frac{1}{2}\Tr(I - \Jk{s}^\parallel \Jk{t}^\parallel - I) \\
    &= \frac{1}{2}\Tr(\Jk{s}^\parallel \Jk{t}^\parallel) \\
    &\geq 0 \text{ because the trace of a positive semi-definite matrix is non negative.}
\end{align}
This proves that $\pmi{s}{t}\geq 0$.
\end{proof}

\paragraph{Proof of claim \ref{claim:3.5}}\label{proof:3.5}
\textit{$\Ip \geq 0$}
\begin{proof}
As per \cref{eq:generalized decomposition}, $\Ip$ can be written as the sum of various $\pmi{s}{t}$ terms, each of which are non-negative by claim \ref{claim:3}.  Therefore $\Ip \geq 0$. 
\end{proof}

\paragraph{Proof of claim \ref{claim:4}}\label{proof:4}
\textit{$\pmi{s}{t} = 0 \text{ if and only if } \Jk{s+t} = \Ak{U}{s+t}^\parallel\Ak{\Sigma}{s+t}\begin{bmatrix}\Ak{V}{s}^T & 0 \\ 0 & \Ak{V}{t}^T\end{bmatrix}$ where $\Ak{U}{s+t}^\parallel$ is semi-orthogonal, $\Ak{V}{s}$ and $\Ak{V}{t}$ are orthogonal and $\Ak{\Sigma}{s+t}$ is diagonal.}
\begin{proof}
Let $A$ be a tall matrix with full rank.  Its singular value decomposition can be written as:
\begin{align}
    A &= U\begin{bmatrix}\Sigma \\ 0\end{bmatrix}V^T \\
    &= \begin{bmatrix}U^\parallel & U^\perp\end{bmatrix}\begin{bmatrix}\Sigma \\ 0\end{bmatrix}V^T \\
    &= U^\parallel \Sigma V^T
\end{align}
$U^\parallel$ is an orthonormal basis for the image of $A$ and $U^\perp$ is an orthonormal basis for the orthogonal complement of the image.  We can write the SVD of $\Jk{s}$ and $\Jk{t}$ as well:
\begin{align}
    \Jk{s} &= \Ak{U}{s}^\parallel \Ak{\Sigma}{s} \Ak{V}{s}^T \\
    \Jk{t} &= \Ak{U}{t}^\parallel \Ak{\Sigma}{t} \Ak{V}{t}^T
\end{align}

Assume that $\pmi{s}{t}=0$.  We must have that $\Jk{s}^T\Jk{t}=0$ because if $\pmi{s}{t} = -\frac{1}{2}\log|I - \Jk{s}^\parallel \Jk{t}^\parallel|=0$, then it must be the case that $\Jk{s}^\parallel \Jk{t}^\parallel=0$, so the images of $\Jk{s}$ and $\Jk{t}$ must be orthogonal.  This means that $\Ak{U}{s}^\parallel$ and $\Ak{U}{t}^\parallel$ are mutually orthogonal because these matrices form orthonormal bases for the images of $\Jk{s}$ and $\Jk{t}$, so their columns form an orthonormal basis for $\Jk{s+t}$.  Next we can write out $\Jk{s+t}$:
\begin{align}
    \Jk{s+t} &= \begin{bmatrix}\Jk{s} & \Jk{t}\end{bmatrix} \label{eq:jacobian decomposition} \\
    &= \begin{bmatrix}\Ak{U}{s}^\parallel \Ak{\Sigma}{s} \Ak{V}{s}^T & \Ak{U}{t}^\parallel \Ak{\Sigma}{t} \Ak{V}{t}^T\end{bmatrix} \\
    &= \underbrace{\begin{bmatrix}\Ak{U}{s}^\parallel & \Ak{U}{t}^\parallel\end{bmatrix}}_{\Ak{U}{s+t}^\parallel}\underbrace{\begin{bmatrix}\Ak{\Sigma}{s} & 0 \\ 0 & \Ak{\Sigma}{t}\end{bmatrix}}_{\Ak{\Sigma}{s+t}}\begin{bmatrix}\Ak{V}{s}^T & 0 \\ 0 & \Ak{V}{t}^T\end{bmatrix} \\
    &= \Ak{U}{s+t}^\parallel\Ak{\Sigma}{s+t}\begin{bmatrix}\Ak{V}{s}^T & 0 \\ 0 & \Ak{V}{t}^T\end{bmatrix}
\end{align}
$\Ak{U}{s+t}^\parallel$ is a semi-orthogonal matrix because all of its columns form an orthonormal basis.

Next assume that $\Jk{s+t}=U\Ak{\Sigma}{s+t}\begin{bmatrix}\Ak{V}{s}^T & 0 \\ 0 & \Ak{V}{s}^T\end{bmatrix}$ where $U$ is semi-orthogonal, $\Ak{\Sigma}{s+t}$ is diagonal and $\Ak{V}{s}$ and $\Ak{V}{t}$ are orthogonal.
\begin{align}
    \Jk{s+t} &= U\Ak{\Sigma}{s+t}\begin{bmatrix}\Ak{V}{s}^T & 0 \\ 0 & \Ak{V}{s}^T\end{bmatrix} \\
    &= \begin{bmatrix}U^\parallel & U^\perp\end{bmatrix}\begin{bmatrix}\Ak{\Sigma}{s} & 0 \\ 0 & \Ak{\Sigma}{t}\end{bmatrix}\begin{bmatrix}\Ak{V}{s}^T & 0 \\ 0 & \Ak{V}{s}^T\end{bmatrix} \\
    &= \begin{bmatrix}U^\parallel \Ak{\Sigma}{s} \Ak{V}{s}^T & U^\perp \Ak{\Sigma}{t} \Ak{V}{t}^T \end{bmatrix} \\
    &= \begin{bmatrix}\Jk{s} & \Jk{t}\end{bmatrix}
\end{align}
$U^\parallel \Ak{\Sigma}{s} \Ak{V}{s}^T$ and $U^\perp \Ak{\Sigma}{t} \Ak{V}{t}^T$ are the SVD of $\Jk{s}$ and $\Jk{t}$ respectively, so $\Jk{s}^\parallel=U^\parallel {U^\parallel}^T$ and $\Jk{t}^\parallel=U^\perp {U^\perp}^T$. Plugging this into claim \ref{claim:2} yields the result $\pmi{s}{t}=0$ because ${U^\parallel}^T U^\perp=0$.
\end{proof}

\paragraph{Proof of claim \ref{claim:4.5}}\label{proof:4.5}
\textit{$\Ip = 0 \text{ if and only if }J = U^\parallel\Sigma\begin{bmatrix}\APk{V}{1}^T & 0 & 0 & 0 \\ 0 & \APk{V}{2}^T & 0 & 0 \\  0 & 0 & \ddots & \vdots \\  0 & 0 & \dots & \APk{V}{|\mathcal{P}|}^T\end{bmatrix}$ where $U^\parallel$ is a semi orthogonal matrix, $\Sigma$ is a diagonal matrix and each $\APk{V}{k}^T$ is an orthogonal matrix with same number of rows and columns as the $k'th$ element of $\mathcal{P}$.}
\begin{proof}
Let $\mathcal{P}'$ be a partition over $[1,\dots,\dim(\z)]$ with $k<|\mathcal{P}|$ elements where the first $k-1$ elements of $\mathcal{P}$ and $\mathcal{P}'$ are identical and the $k$'th element of $\mathcal{P}'$ is the union of the final $|\mathcal{P}|-k$ elements of $\mathcal{P}$.  We will use $\mathcal{P}_k$ to denote the $k$'th element of $\mathcal{P}$, $\mathcal{P}_{:k}$ to denote the union of the first $k$ elements of $\mathcal{P}$ and $\mathcal{P}_{k:}$ to denote the union of the $k$'th to last elements of $\mathcal{P}$.   We will use a proof by induction to prove one direction of the claim where we assume that $\Ip=0$.

The base case is when $\mathcal{P}'$ contains only $\mathcal{P}_{1}$ and $\mathcal{P}_{2:}$.  From \cref{subsection:change of variables decomp} we know that we can construct the partition using a tree that has a parent node equal to $\mathcal{P}'$ with children that are $\mathcal{P}_{1}$ and $\mathcal{P}_{2:}$.  This means $\Ip$ will be the sum of $\mathcal{I}_{1,2:}$ and other $\mathcal{I}$ terms and because we assumed that $\Ip=0$, it must be that $\mathcal{I}_{1,2:}=0$, so we can apply claim \ref{claim:4} to satisfy the inductive hypothesis.

Next assume $\mathcal{P}'$ contains the first $k-1$ elements of $\mathcal{P}$ and an element containing the union of the remainder of $\mathcal{P}$.  Assuming that the inductive hypothesis is true, we can write the Jacobian matrix as
\begin{align}
    J &= U^\parallel\Sigma\begin{bmatrix}\APk{V}{:k-1}^T & 0 \\  0 & \APk{V}{k:}^T\end{bmatrix} \quad \text{where} \\
    \APk{V}{:k-1}^T &= \begin{bmatrix}\APk{V}{1}^T & 0 & 0 & 0 \\ 0 & \APk{V}{2}^T & 0 & 0 \\  0 & 0 & \ddots & \vdots \\  0 & 0 & \dots & \APk{V}{k-1}^T\end{bmatrix}
\end{align}

We can rewrite $J$ to isolate the columns in the $\mathcal{P}_k$ partition:
\begin{align}
    J &= U^\parallel\Sigma\begin{bmatrix}\APk{V}{:k-1}^T & 0 \\  0 & \APk{V}{k:}^T\end{bmatrix} \label{eq:62} \\
    &= \begin{bmatrix}\APk{U}{:k-1}^\parallel & \APk{U}{k:}^\parallel\end{bmatrix}\begin{bmatrix}\APk{\Sigma}{:k-1} & 0 \\ 0 & \APk{\Sigma}{k:}\end{bmatrix}\begin{bmatrix}\APk{V}{:k-1}^T & 0 \\  0 & \APk{V}{k:}^T\end{bmatrix} \\
    &= \begin{bmatrix}\APk{U}{:k-1}^\parallel\APk{\Sigma}{:k-1}\APk{V}{:k-1}^T & \APk{U}{k:}^\parallel\APk{\Sigma}{k:}\APk{V}{k:}^T\end{bmatrix}
\end{align}
Next, let $\APk{J}{k:} = \APk{U}{k:}^\parallel\APk{\Sigma}{k:}\APk{V}{k:}^T$.  $\APk{J}{k:}$ contains the columns of $J$ with indices in the final $|\mathcal{P}-k|$ elements of $\mathcal{P}$.  Choose a partition from these final elements, $\mathcal{P}_{k}$.   Let $\APk{J}{k}$ contain the columns of $\APk{J}{k:}$ that are in $\mathcal{P}_{k}$ and let $\APk{J}{k+1:}$ contain the remaining columns.  Because $\mathcal{P}_{k}\in \mathcal{P}$, it must be true that $\mathcal{I}_{\mathcal{P}_{k},\mathcal{P}_{k+1:}}=0$.  Therefore we can apply claim \ref{claim:4} to decompose $\APk{J}{k}$
\begin{align}
    \APk{J}{k:} &= \APk{U}{k:}^\parallel\APk{\Sigma}{k:}\APk{V}{k:}^T \\
    &= \APk{U}{k:}^\parallel\APk{\Sigma}{k:}\begin{bmatrix}\APk{V}{k}^T & 0 \\ 0 & \APk{V}{k+1:}^T\end{bmatrix}
\end{align}
Plugging this back into Eq.\ref{eq:62} and yields
\begin{align}
    J &= U^\parallel\Sigma\begin{bmatrix}\APk{V}{:k-1}^T & 0 \\  0 & \APk{V}{k:}^T\end{bmatrix} \\
    &= U^\parallel\Sigma\begin{bmatrix}\APk{V}{:k-1}^T & 0 & 0 \\ 0 & \APk{V}{k}^T & 0 \\ 0 & 0 & \APk{V}{k+1:}^T\end{bmatrix} \\
    &= U^\parallel\Sigma\begin{bmatrix}\APk{V}{:k}^T & 0 \\  0 & \APk{V}{k+1:}^T\end{bmatrix}
\end{align}
So by induction, $J = U^\parallel\Sigma\begin{bmatrix}\APk{V}{1}^T & 0 & 0 & 0 \\ 0 & \APk{V}{2}^T & 0 & 0 \\  0 & 0 & \ddots & \vdots \\  0 & 0 & \dots & \APk{V}{|\mathcal{P}|}^T\end{bmatrix}$.

For the other direction, assume that $J = U^\parallel\Sigma\begin{bmatrix}\APk{V}{1}^T & 0 & 0 & 0 \\ 0 & \APk{V}{2}^T & 0 & 0 \\  0 & 0 & \ddots & \vdots \\  0 & 0 & \dots & \APk{V}{|\mathcal{P}|}^T\end{bmatrix}$.  Clearly $J^TJ$ will be a block diagonal matrix, so $\log|J^TJ| = \sum_{\setstyle{k}\in \mathcal{P}}\log|\Jk{k}^T\Jk{k}|$.  It trivially follows from claim \ref{claim:1.5} that $\Ip=0$.
\end{proof}

\paragraph{Proof of claim \ref{claim:5}}\label{proof:5}
\textit{$\pmi{s}{t} = 0$ if and only if $\fcontour{s}$ and $\fcontour{t}$ intersect orthogonally.}
\begin{proof}
We saw in the proof of claim \ref{claim:4} that the image of $\Jk{s}$ and $\Jk{t}$ are orthogonal when $\pmi{s}{t} = 0$.  At the point that $\fcontour{s}$ and $\fcontour{t}$ intersect, they are aligned with the images of $\Jk{s}$ and $\Jk{t}$ respectively, so they will intersect orthogonally.  Similarly,if $\fcontour{s}$ and $\fcontour{t}$ intersect orthogonally, their definition tells us that the image of $\Jk{s}$ and $\Jk{t}$ are orthogonal, so we must have $\pmi{s}{t}=0$.
\end{proof}

\paragraph{Proof of claim \ref{claim:6}}\label{proof:6}
\textit{$\pmihat{s}{t} = \frac{1}{2}\log|\Gk{s+t}\Gk{s+t}^T| - \frac{1}{2}\log|\Gk{s}\Gk{s}^T| - \frac{1}{2}\log|\Gk{t}\Gk{t}^T|$}
\begin{proof}
\begin{align}
    \pmihat{s}{t} &= \Lhatk{s+t} - \Lhatk{s} - \Lhatk{t} \\
    &= \underbrace{\log \frac{p_{\setstyle{s}+\setstyle{t}}(\zk{s+t})}{p_{\setstyle{s}}(\zk{s})p_{\setstyle{t}}(\zk{t})}}_{=0\text{ by assumption of how prior factors}} + \frac{1}{2}\log|\Gk{s+t}\Gk{s+t}^T| - \frac{1}{2}\log|\Gk{s}\Gk{s}^T| - \frac{1}{2}\log|\Gk{t}\Gk{t}^T|\\
    &= \frac{1}{2}\log|\Gk{s+t}\Gk{s+t}^T| - \frac{1}{2}\log|\Gk{s}\Gk{s}^T| - \frac{1}{2}\log|\Gk{t}\Gk{t}^T|
\end{align}
\end{proof}

\paragraph{Proof of claim \ref{claim:6.5}}\label{proof:6.5}
\textit{$\Ihatp = -\frac{1}{2}\log|GG^T| + \frac{1}{2}\sum_{\setstyle{k}\in \mathcal{P}}\log|\Gk{k}\Gk{k}^T|$}
\begin{proof}
\begin{align}
    \Ihatp &= \widehat{\mathcal{L}} - \sum_{\setstyle{k}\in \mathcal{P}}\Lhatk{k} \\
    &= \underbrace{\log \frac{p_\z(\z)}{\prod_{\setstyle{k}\in \mathcal{P}}p_\smallsetstyle{k}(\z_\smallsetstyle{k})}}_{=0\text{ by assumption of how prior factors}} + \frac{1}{2}\log|GG^T| - \frac{1}{2}\sum_{\setstyle{k}\in \mathcal{P}}\log|\Gk{k}\Gk{k}^T|\\
    &= \frac{1}{2}\log|GG^T| - \frac{1}{2}\sum_{\setstyle{k}\in \mathcal{P}}\log|\Gk{k}\Gk{k}^T|
\end{align}
\end{proof}

\paragraph{Proof of claim \ref{claim:7}}\label{proof:7}
\textit{$\pmihat{s}{t} = \frac{1}{2}\log|I - \Gk{s}^\parallel \Gk{t}^\parallel|$.}
\begin{proof}
\begin{align}
    |\Gk{s+t}\Gk{s+t}^T| &= |\begin{bmatrix}\Gk{s} \\ \Gk{t}\end{bmatrix}\begin{bmatrix}\Gk{s}^T  & \Gk{t}^T\end{bmatrix}| \\
    &= |\begin{bmatrix}\Gk{s}\Gk{s}^T & \Gk{s}\Gk{t}^T  \\ \Gk{t}\Gk{s}^T  & \Gk{t}\Gk{t}^T\end{bmatrix}| \\
    &= |\Gk{s}\Gk{s}^T||\Gk{t}\Gk{t}^T - \Gk{t}\underbrace{\Gk{s}^T(\Gk{s}\Gk{s}^T)^{-1}\Gk{s}}_{\Gk{s}^\parallel}\Gk{t}^T|  \\
    &= |\Gk{s}\Gk{s}^T||\Gk{t}\Gk{t}^T||I - \Gk{s}^\parallel \underbrace{\Gk{t}^T(\Gk{t}\Gk{t}^T)^{-1}\Gk{t}}_{\Gk{t}^\parallel}| \\
    &= |\Gk{s}\Gk{s}^T||\Gk{t}\Gk{t}^T||I - \Gk{s}^\parallel\Gk{t}^\parallel|    
\end{align}
Therefore $\frac{1}{2}\log|\Gk{s+t}\Gk{s+t}^T| - \frac{1}{2}\log|\Gk{s}\Gk{s}^T| - \frac{1}{2}\log|\Gk{t}\Gk{t}^T| = \frac{1}{2}\log|I - \Gk{s}^\parallel\Gk{t}^\parallel|$.  An application of claim \ref{claim:6} completes the proof that $\pmihat{s}{t} = \frac{1}{2}\log|I - \Gk{s}^\parallel \Gk{t}^\parallel|$.
\end{proof}

\paragraph{Proof of claim \ref{claim:8}}\label{proof:8}
\textit{$\pmihat{s}{t} \leq 0$}
\begin{proof}
Notice that we can prove that $-\frac{1}{2}\log|I - \Gk{s}^\parallel \Gk{t}^\parallel| \geq 0$ using an identical proof as the one used to prove \ref{claim:3}.  Therefore it must be that $\pmihat{s}{t}=\frac{1}{2}\log|I - \Gk{s}^\parallel \Gk{t}^\parallel| \leq 0$.
\end{proof}

\paragraph{Proof of claim \ref{claim:8.5}}\label{proof:8.5}
\textit{$\Ihatp \leq 0$}
\begin{proof}
The same steps used in \cref{eq:generalized decomposition} to write $\Ip$ as the sum of various $\pmi{s}{t}$ terms can be used to write $\Ihatp$ as the sum of various $\pmihat{s}{t}$ terms.   Since each $\pmihat{s}{t} \leq 0$, it must be that $\Ihatp \leq 0$.
\end{proof}

\paragraph{Proof of claim \ref{claim:9}}\label{proof:9}
\textit{$\pmihat{s}{t} = 0 \text{ if and only if } \Gk{s+t} = \begin{bmatrix}\Ak{V}{s} & 0 \\ 0 & \Ak{V}{t}\end{bmatrix}\Ak{\Sigma}{s+t}{\Ak{U}{s+t}^\parallel}^T$ where ${\Ak{U}{s+t}^\parallel}^T$ is semi-orthogonal, $\Ak{V}{s}$ and $\Ak{V}{t}$ are orthogonal and $\Ak{\Sigma}{s+t}$ is diagonal.}
\begin{proof}
The proof is identical to that of \ref{claim:4} except that the matrices are transposed.
\end{proof}

\paragraph{Proof of claim \ref{claim:10}}\label{proof:10}
\textit{$\Ihatp = 0 \text{ if and only if }J = \begin{bmatrix}\APk{V}{1} & 0 & 0 & 0 \\ 0 & \APk{V}{2} & 0 & 0 \\  0 & 0 & \ddots & \vdots \\  0 & 0 & \dots & \APk{V}{|\mathcal{P}|}\end{bmatrix}\Sigma{U^\parallel}^T$ where ${U^\parallel}^T$ is a semi orthogonal matrix, $\Sigma$ is a diagonal matrix and each $\APk{V}{k}$ is an orthogonal matrix with same number of rows and columns as the $k'th$ element of $\mathcal{P}$.}
\begin{proof}
The proof is identical to that of \ref{claim:4.5} except that the matrices are transposed.
\end{proof}

\paragraph{Proof of claim \ref{claim:11}}\label{proof:11}
\textit{If $\dim(\x)=\dim(\z)$, then $\Ip = 0$ if and only if $\Ihatp = 0$}
\begin{proof}
First assume that $\Ip=0$.  Then by \ref{claim:4}, $J=U\Sigma V^T$ where $U$ is orthogonal, $\Sigma$ is diagonal and $V^T$ is a block diagonal matrix with orthogonal blocks.  Because $\dim(\x)=\dim(\z)$, $G=J^{-1}=V\Sigma U^T$.  Then by \ref{claim:9}, $\Ihatp=0$.  The reverse clause is proven in the same manner.  Assuming $\Ihatp=0$, we can use \ref{claim:9} to decompose $G$, take its inverse and use \ref{claim:4} to prove $\Ip=0$.
\end{proof}

\section{PF Proofs}

\begin{lemma}\label{lemma:principal components flow}
Each contour of a PF at $\x$ is spanned by a unique set of principal components.
\end{lemma}
\begin{proof}
The principal components of a flow at $\x$ are the eigenvectors of $JJ^T$.  From claim \ref{claim:4.5} we know that $J=U^\parallel\Sigma\begin{bmatrix}\APk{V}{1}^T & 0 & 0 & 0 \\ 0 & \APk{V}{2}^T & 0 & 0 \\  0 & 0 & \ddots & \vdots \\  0 & 0 & \dots & \APk{V}{|\mathcal{P}|}^T\end{bmatrix}$ where $\mathcal{P}_{k}$ is the $k'th$ element of the partition of the latent space.  The $U^\parallel$ and $\Sigma$ matrices form an eigendecomposition of $JJ^T$ because $JJ^T=U^\parallel\Sigma^2 {U^\parallel}^T$, so the columns of $U^\parallel$ are the principal components of the PF.  Next, we can rewrite $J$ so that the dependence of each contour on the principal components is explicit:
\begin{align}
    J &= \begin{bmatrix}\APk{U^\parallel}{1}\APk{\Sigma}{1}\APk{V}{1}^T & \APk{U^\parallel}{2}\APk{\Sigma}{2}\APk{V}{2}^T & \dots & \APk{U^\parallel}{|\mathcal{P}|}\APk{\Sigma}{|\mathcal{P}|}\APk{V}{|\mathcal{P}|}^T\end{bmatrix}
\end{align}
$\APk{U^\parallel}{k}\APk{\Sigma}{k}\APk{V}{k}^T$ is the $k$'th contour of the PF.  We can clearly see that its image is spanned by the principal components with indices in $\mathcal{P}_k$.  Because $\mathcal{P}_i\bigcap \mathcal{P}_j = \emptyset, \forall i,j$ we conclude that each contour of a PF at $\x$ is spanned by a unique set of principal components.
\end{proof}

\begin{theorem}\label{appendix:principal manifolds flow theorem}
The contours of a principal manifold flow are principal manifolds.
\end{theorem}
\begin{proof}
The principal manifold of a flow is found by integrating along the direction of a principal component.
\begin{align}
    \frac{d\x(t)}{dt} &= \wk{k}(\x(t))
\end{align}
\cref{lemma:principal components flow} tells us that the contours of a PF are locally spanned by the principal components and that the eigenvectors and eigenvalues of $JJ^T$ are equal to $U^\parallel$ and $\Sigma^2$ respectively.  So we can simplify by letting $\x(t)=f(\z(t))$.
% \begin{align}
%     \frac{d\x(t)}{dt} &= \wk{k}(\x(t)) \\
%     \frac{df(\z)}{d\z}\frac{d\z(t)}{dt} &= \wk{k}(f(\z(t)) \\
%     \frac{d\z(t)}{dt} &= \frac{df(\z)}{d\z}^{-1}\frac{df(\z)}{d\zk{k}} \\
%                      &= e_{\setstyle{k}}
% \end{align}
\begin{align}
    \frac{d\x(t)}{dt} &= \wk{k}(\x(t)) \\
    J\frac{d\z(t)}{dt} &= \APk{U^\parallel}{k}\APk{\Sigma}{k} \\
    \frac{d\z(t)}{dt} &= J^+\APk{U^\parallel}{k}\APk{\Sigma}{k} \\
                     &= \begin{bmatrix}0 & \dots & \APk{V}{k}\APk{\Sigma}{k}\APk{V}{k}^T & \dots & 0\end{bmatrix}^T
\end{align}
This derivation tells us that the principal manifold, when traced out in the latent space, is equal to a manifold that only varies along dimensions in $\setstyle{k}$.  This is exactly how contours are generated, therefore the principal manifolds of a PF are its contours.
\end{proof}

\newpage
\section{Additional details on experiments}\label{appendix:extended results}

\subsection{2D experiments}\label{appendix:2d experiments}

\begin{figure}
    \centering
    \includegraphics[width=\textwidth]{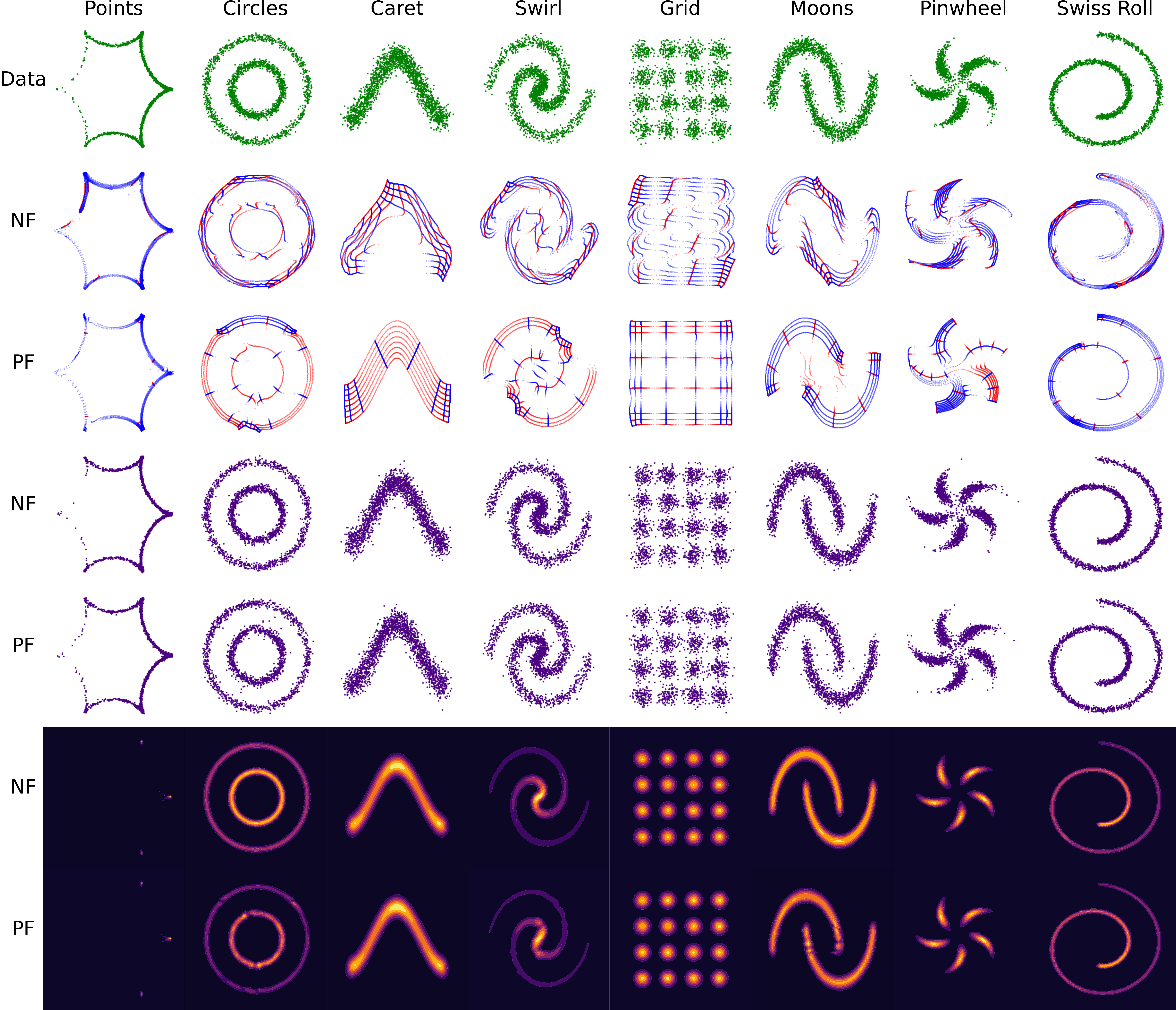}
    \caption{Extended results for synthetic datasets.  Top row has true samples, the next two rows are contours, the next two are samples and the last two are probability densities computed by the models.}
    \label{fig:extended contour2d}
\end{figure}

See Fig.\ref{fig:extended contour2d} for extended results.  Each dataset was generated with 1,000,000 data points and split into 700,000 for training and 300,000 for testing.  As mentioned in the main text, the architecture used on all of the datasets consisted of 10 coupling layers with logistic mixture cdf layers that used 8 mixture components and an affine coupling layer that shared the same conditioner network.  Each conditioner consisted of 5 residual layers with a hidden layer size of 64.  The models were trained using the AdaBelief \citep{zhuang2020adabelief} optimization algorithm with a learning rate of $1\times 10^{-3}$ and a batch size of 256, and $\alpha=10.0$.  Each model was trained for approximately 4 hours on either a NVIDIA 1080ti or 2080ti.

\subsection{Variable dimension manifold}\label{appendix:3d experiments}
We used a flow with 20 coupling based neural spline \citep{durkan_neural_2019} layers, each with 8 knot points, followed directly by an affine coupling layer that is parametrized by the same conditioner network as done in \citep{ho_flow_2019}.  The conditioner networks all consisted of a 5 layer residual network with a hidden dimension of 32.  We used a unit Gaussian prior.  The model was trained for around 4 hours on a NVIDIA 3090 gpu with a learning rate of $1\times 10^{-4}$ with the AdaBelief \citep{zhuang2020adabelief} optimization algorithm and a batch size of 2048 and $\alpha=5.0$.  We trained on 2,100,000 data points and evaluated on 900,000.  As stated in the main text, the data was augmented with Gaussian noise with a standard deviation of 0.01 to ensure that the model did not collapse during training.  The true generative model for the data is:
\begin{align}
    z_1 &\sim \frac{1}{3}(N(-2, 0.3) + N(0, 0.3) + N(2, 0.3)) \\
    z_2 &\sim N(0,1) \\
    x &= f(z_1,z_2) = \begin{bmatrix}z_1 \\ z_2\text{max}(0,1-|\frac{1}{z_1}|) \\ \sin(z_1) \end{bmatrix}
\end{align}
The true density was computed in a piecewise manner.  If $x_1 = 0$, then
\begin{align}
    p(x) = p(z_1)|\frac{df(z)}{dz_1}^T\frac{df(z)}{dz_1}|^{\frac{-1}{2}}
\end{align}
Otherwise, 
\begin{align}
    p(x) = p(z_1)p(z_2)|\frac{df(z)}{dz}^T\frac{df(z)}{dz}|^{\frac{-1}{2}}
\end{align}
The Jacobian determinants were computed with automatic differentiation.

In \cref{fig:large manifold densities} we show a larger version of \cref{fig:manifold densities}, in \cref{fig:manifold samples} we showcase samples pulled from the PF and the contours learned by the model.  In \cref{fig:manifold error histogram} we see that the PF does extremely well in predicting the log likelihood of the test set.  The final KL divergence between the true data distribution and learned was 0.0146.  To choose the rank at test time, we compared the three contour likelihoods provided by the model and filtered out the likelihoods that were negligible compared to the others.

\begin{figure}
    \centering
    \includegraphics[width=\textwidth]{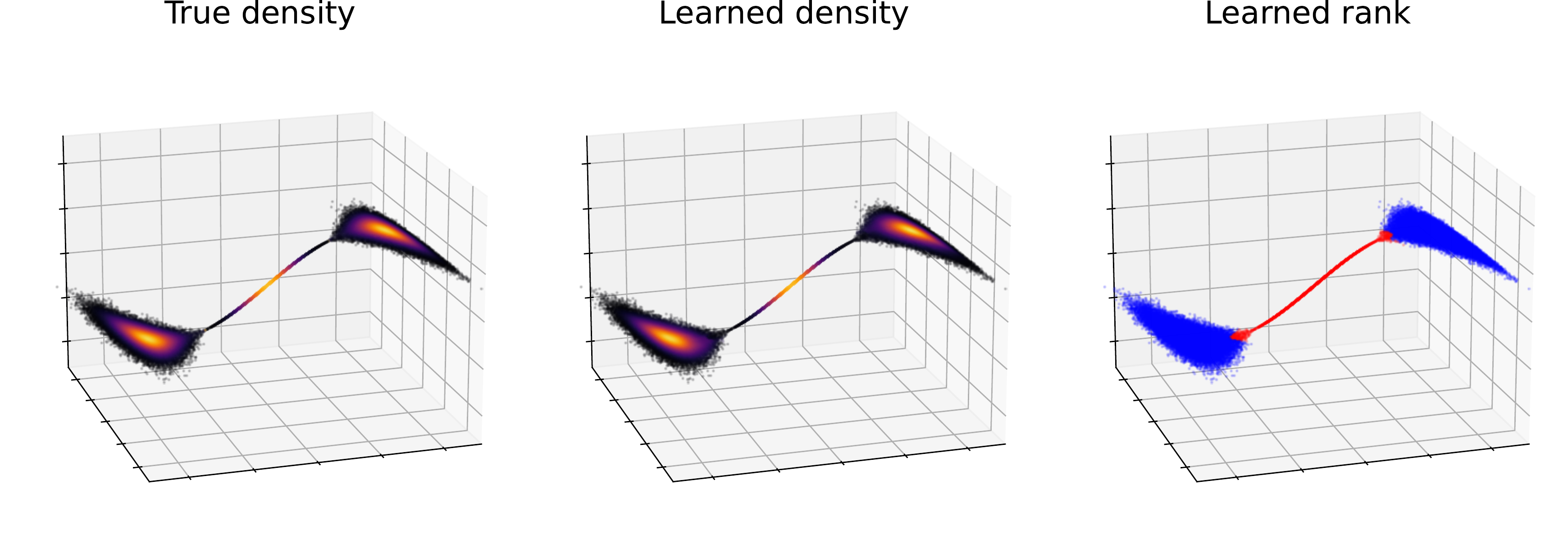}
    \caption{A larger version of \cref{fig:manifold densities}}
    \label{fig:large manifold densities}
\end{figure}

\begin{figure}
    \centering
    \includegraphics[width=\textwidth]{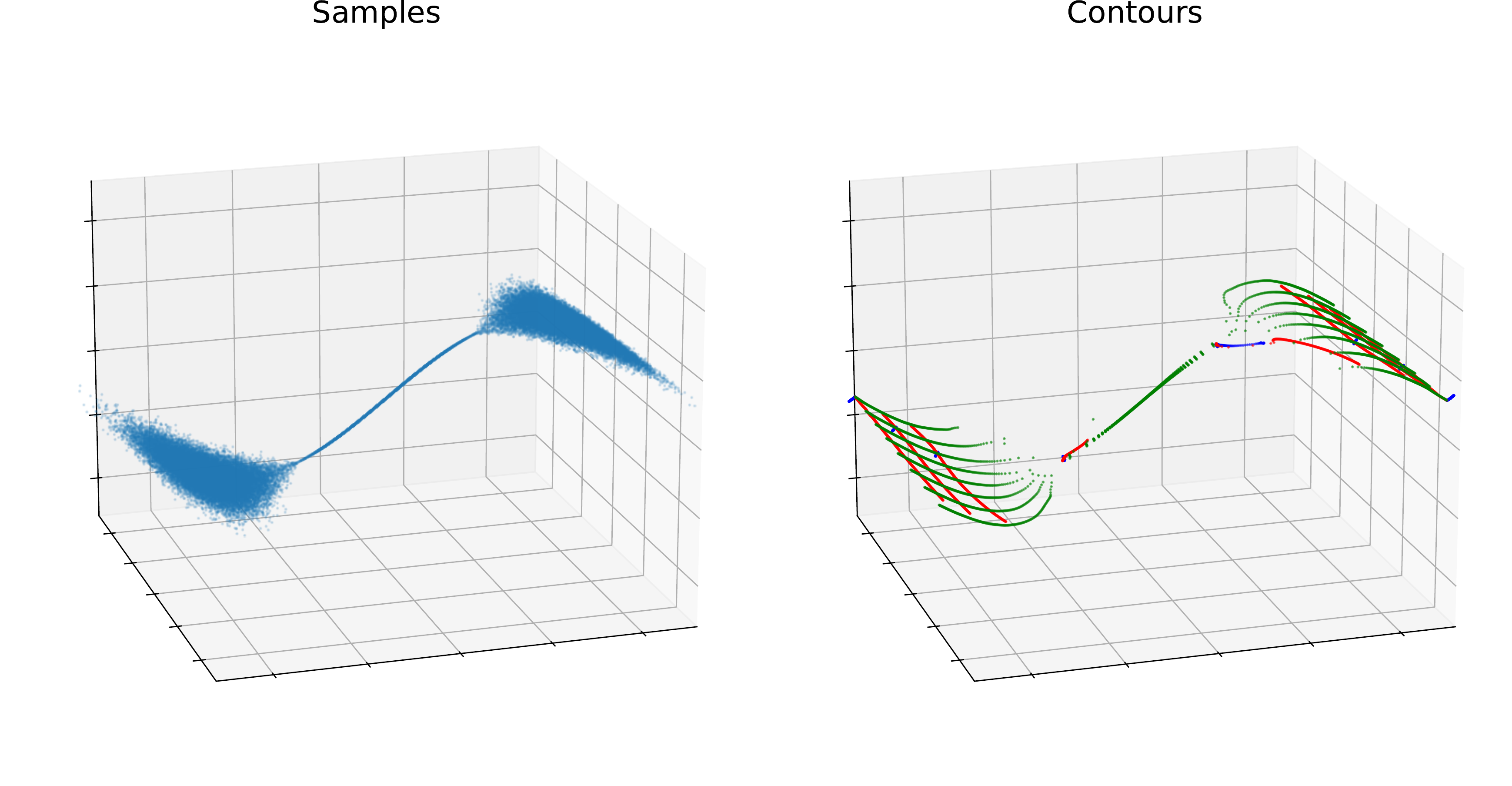}
    \caption{Samples and contours from the model trained for \cref{sec:manifold 3d}}
    \label{fig:manifold samples}
\end{figure}

\begin{figure}
    \centering
    \includegraphics[width=0.5\textwidth]{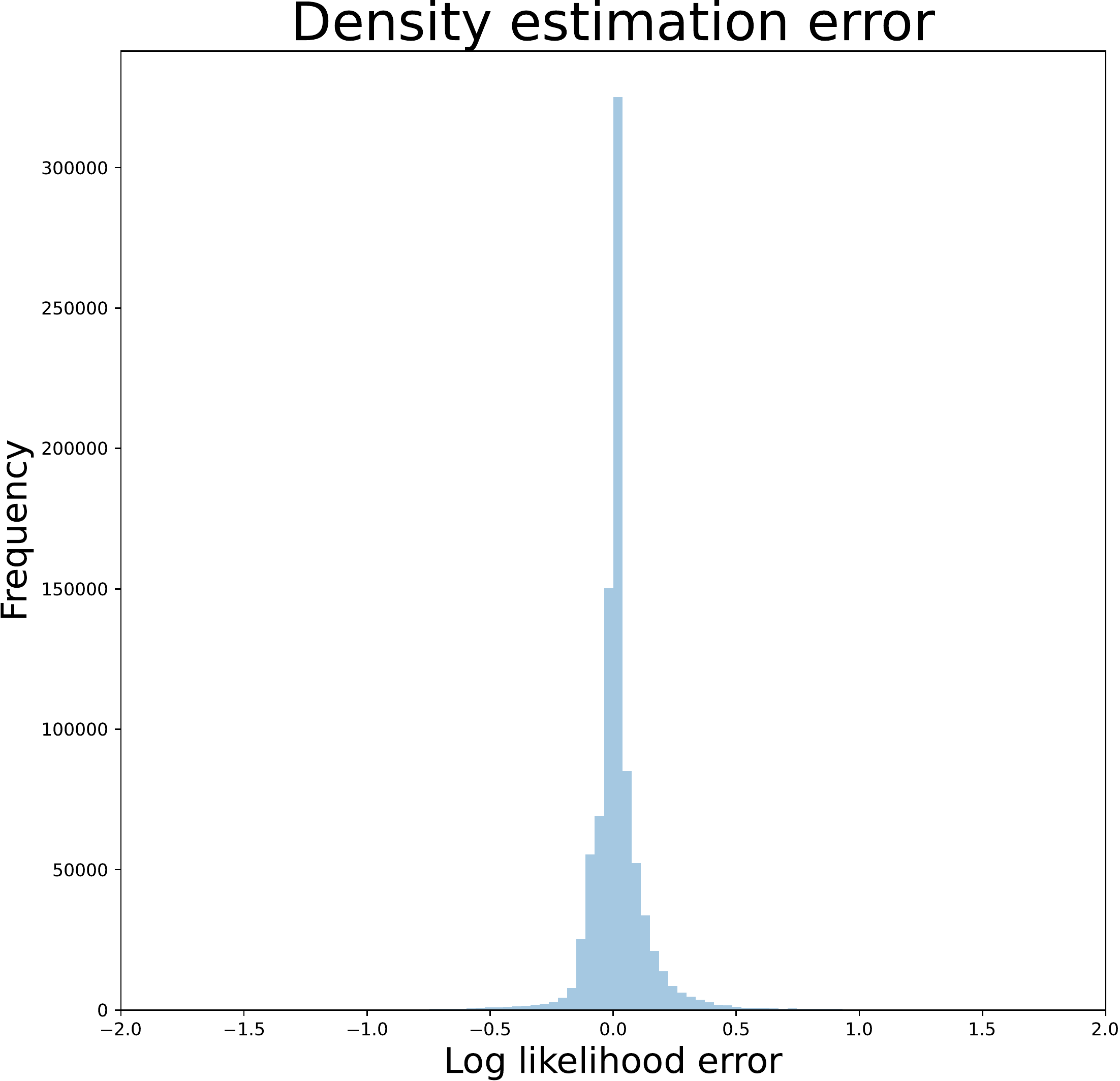}
    \caption{Histogram of the difference between the true log likelihood and the predicted for the experiment in \cref{sec:manifold 3d}}
    \label{fig:manifold error histogram}
\end{figure}

\subsection{iPF}\label{appendix:mnist iPF}

\paragraph{Preprocessing}
For training, we preprocessed incoming batches of data using uniform dequantization and a scaling layer + logit transformation as described in \citep{dinh_density_2017}.  Then each $(28\times 28\times 1)$ image was flattened into a 784 dimensional vector.

\paragraph{Model architectures}
The iPF and iNF architectures were composed of two parts.  The first is a flow in the full 784 dimensional ambient space that consisted of 20 layers of GLOW \citep{kingma_glow_2018} with each conditioner network consisting of 3 residual networks with a hidden dimension of 32 and dropout rate of 0.2.  After the GLOW layers, the output was sliced so that the resulting dimensionality was 10 and this low dimensional vector was passed to a unit Gaussian prior.

The second part of the architecture, used during fine-tuning, consisted of 10 coupling based neural spline layers with 8 knots and affine coupling layers.  Each conditioner contained 4 residual network layers with a hidden dimension of 4.  The input to this flow is the 10 dimensional output of the first component and the output is fed into a unit Gaussian prior.

\paragraph{Training}
The overall model was trained in two stages.  The first stage optimized the objective in section 4 of \citep{caterini_rectangular_2021} using only the GLOW layers.  The objectives we optimized were:
\begin{align}
    \text{Objective1}_{\text{iPF}} &= \sum_{\x\in \mathcal{D}} -\log p_\z(g(\x)) + \frac{\dim(\z)}{2}\log(\sum_iJ_{ki}(g(\x))^2) + \gamma ||f(g(\x)) - x||^2, \quad k\sim \text{Uniform}(1,\dots,10)
\end{align}
\begin{align}
    \text{Objective1}_{\text{iNF}} &= \sum_{\x\in \mathcal{D}} -\log p_\z(g(\x)) + \frac{1}{2}\log|J(g(\x))^TJ(g(\x))| + \gamma ||f(g(\x)) - x||^2
\end{align}
For both models we set $\gamma=10$, used a batch size of 64, learning rate of $1\times 10^{-4}$ and the AdaBelief \citep{zhuang2020adabelief} optimization algorithm.  We found that it was crucial to use a small learning rate, otherwise training would fail.  These models were trained for approximately 36 hours on either a NVIDIA 3090ti or RTX8000 gpu.

After this stage of training, we combined the GLOW layers with the neural spline layers into one normalizing flow.  We then froze the parameters for the GLOW layers and trained the parameters of the spline layers using a learning rate of $1\times 10^{-3}$ for another 24 hours on the same objective as before, but without the reconstruction error term:
\begin{align}
    \text{Objective2}_{\text{iPF}} &= \sum_{\x\in \mathcal{D}} -\log p_\z(g(\x)) + \frac{\dim(\z)}{2}\log(\sum_iJ_{ki}(g(\x))^2), \quad k\sim \text{Uniform}(1,\dots,10)
\end{align}
\begin{align}
    \text{Objective2}_{\text{iNF}} &= \sum_{\x\in \mathcal{D}} -\log p_\z(g(\x)) + \frac{1}{2}\log|J(g(\x))^TJ(g(\x))|
\end{align}

\begin{figure}
    \centering
    \includegraphics[width=\textwidth]{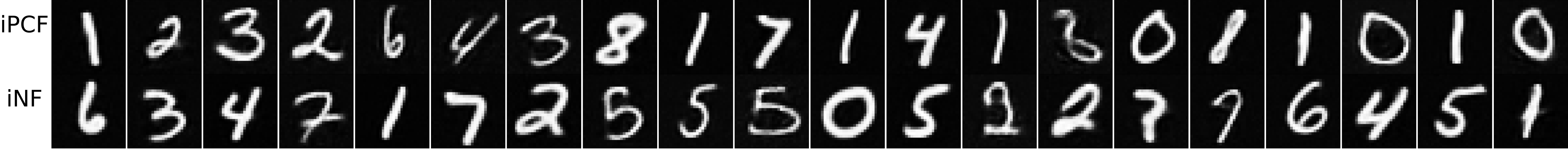}
    \caption{Random samples from the iPF and injective normalizing flow trained in \cref{section:iPF experiment}}
    \label{fig:iPF samples}
\end{figure}

\section{Practical considerations}\label{appendix:practical considerations}
PFs have many nice theoretical properties, but can be difficult to train and interpret in practice.  We find that the constraint $\Ip=0$ can only be satisfied with normalizing flows that are very expressive.  We conjecture that the reason is because the constraint $\Ip=0$ requires that each $O(2^{|\mathcal{P}|})$ possible contour that can be constructed are orthogonal to the other contours.  As a result, we find it necessary to keep $|\mathcal{P}|$ small by using iPFs to model high dimensional data, increasing the size of each partition or using a feature extractor flow that can transform data into a simpler form for the PF.  Furthermore, the latent space of PFs is not trivial to interpret.  While it is true that the latent variables of PFs correspond to different principal manifolds, the index of the dimension corresponding to different principal manifolds can change (see \cref{fig:contour2d} for clear examples of this).  This means that a smooth path through over a principal manifold may require a discontinuous path through the latent space.

\end{document}